\documentclass[letterpaper]{article}
\usepackage{arxiv}
\usepackage{times}
\usepackage{helvet}
\usepackage{courier}
\usepackage[hyphens]{url}
\usepackage{graphicx}
\usepackage{natbib}
\usepackage{amsmath,amssymb,amsthm,mathtools}
\usepackage{booktabs}
\usepackage{multirow}
\usepackage{tikz}
\usepackage{pgfplots}
\pgfplotsset{compat=1.17}
\usepackage{algorithm}
\usepackage{algpseudocode}
\usetikzlibrary{arrows.meta,positioning}
\newtheorem{theorem}{Theorem}
\newtheorem{lemma}[theorem]{Lemma}
\newtheorem{proposition}[theorem]{Proposition}
\newtheorem{corollary}[theorem]{Corollary}
\newtheorem{assumption}{Assumption}

\newtheorem{remark}{Remark}

\newcommand{\E}{\mathbb{E}}
\newcommand{\Scal}{\mathcal{S}}
\newcommand{\Acal}{\mathcal{A}}
\newcommand{\Ccal}{\mathcal{C}}
\newcommand{\dann}{\hat d}
\newcommand{\dup}{\hat d^{+}}
\newcommand{\dinit}{d_0}
\newcommand{\eop}{\varepsilon_{\mathrm{op}}}
\newcommand{\Dg}{D_\gamma}
\newcommand{\rhoup}{\rho^{+}}
\newcommand{\norm}[1]{\left\lVert#1\right\rVert}

\title{Sub-Quadratic Bisimulation Metrics via Approximate Nearest Neighbors:\\
Coverage-Augmented Guarantees and Computable Two-Sided Certificates}
\author{
    Ibne Farabi Shihab\textsuperscript{\rm 1}\equalcontrib\thanks{Corresponding author.},
    Joyanta Jyoti Mondal\textsuperscript{\rm 2}\equalcontrib
}
\affiliations{
\textsuperscript{\rm 1}Department of Computer Science, Iowa State University, USA\\
    \textsuperscript{\rm 2}Department of Computer and Information Sciences, University of Delaware, USA\\

    ishihab@iastate.edu, joyanta@udel.edu 
}

\begin{document}
\maketitle

\begin{abstract}
Bisimulation metrics quantify behavioral similarity in Markov decision processes, but their
Wasserstein fixed-point operator updates every state pair and incurs quadratic pairwise work. We
give a certificate-carrying sub-quadratic method for MDPs with bounded transition support and a
useful low-dimensional indexing representation: an approximate-nearest-neighbor index selects the
pairs updated by the exact restricted operator, while monotone lower and upper runs enclose the
exact metric at every sweep. The main analytical result is a coverage-augmented anytime bound:
local index quality alone cannot control global error, because uncovered pairs retain their
initialization gap. The limiting error is at most $\max(\rho,\eop/(1-\gamma))$, and with exact
covered backups the lower arm satisfies $\|\dann-d\|_\infty=\rho$. Because $\rho$ depends on the
unknown exact metric, the algorithm returns the observable sandwich width instead; agreement of
the induced lower and upper clusterings certifies exact recovery of the covered aggregation. A
reward-oblivious lower bound shows sub-quadratic index-first coverage cannot remove the coverage
term, while a separate adaptive lower bound requires $\Omega(|\Scal|)$ pair evaluations.
Exact-operator experiments verify the identity and enclosure in every seeded run, and timing
experiments recover quadratic versus sub-quadratic scaling under both cheap and full Wasserstein
backups. On the grouped $|\Scal|=64$ benchmark, exact restricted refinement reaches the
exact-metric skyline once retrieval covers roughly half of all pairs, while independently trained
MICo and DBC baselines stay $22$--$33\times$ above that skyline at every retrieval budget. Taxi
shows the certificate abstaining under an uninformative embedding, while a $2500$-state gridworld
improves over a reward-only metric by $28.6\%$ using $12.8\%$ of one quadratic sweep.
\end{abstract}

\section{Introduction}
A reinforcement learning agent that treats two behaviorally identical states as identical can
generalize across them and learn from fewer samples; the bisimulation metric is the formal device
licensing this. Introduced for finite MDPs by \citet{ferns2004metrics} and extended to continuous
and infinite spaces by \citet{ferns2011bisimulation}, the metric $d$ is small exactly when two
states are behaviorally interchangeable: the fixed point of an operator coupling immediate-reward
difference with Wasserstein distance between next-state distributions. States at distance zero are
bisimilar in the sense of \citet{larsen1989bisimulation,givan2003equivalence}; the metric degrades
gracefully between equivalence classes, useful for state abstraction
\citep{li2006abstraction,ravindran2004algebraic} and representation learning
\citep{zhang2021dbc,gelada2019deepmdp,castro2021mico}.

The price is quadratic: each application visits $\Theta(|\Scal|^2)$ pairs and solves an
optimal-transport problem per pair, prohibitive at the Atari-scale state spaces motivating deep
representation learning \citep{bellemare2013ale,mnih2015dqn}. The standard response abandons the
exact metric for a sampled surrogate folded into a learned embedding
\citep{castro2020scalable,castro2021mico,zhang2021dbc}, trading guarantees for scalability. We ask
instead whether the metric can be computed sub-quadratically, keeping a provable, \emph{checkable}
relation to the exact object.

The opportunity arises in MDPs whose bisimulation geometry admits a low-dimensional Euclidean
representation: metric-embedding theory and landmark MDS
\citep{bourgain1985embedding,linial1995geometry,cox2008mds} give a practical index, and given
$\norm{\phi(s)-\phi(s')}\approx d(s,s')$, approximate-nearest-neighbor search retrieves candidates
sub-linearly \citep{indyk1998ann,andoni2008nearoptimal,malkov2020hnsw}, restricting the sweep to
$O(|\Scal|k')$ pairs (Figure~\ref{fig:overview}). This isn't assumed universally: a poor embedding
still leaves the certificate valid, exposing rather than hiding failure.
Theorem~\ref{thm:complexity} makes the complexity precise; the harder question is what error sparse
coverage necessarily leaves.

\begin{figure*}[t]
  \centering
  \begin{tikzpicture}[>=Stealth, font=\footnotesize,
    box/.style={draw, rounded corners, align=center, minimum height=7mm,
      minimum width=15mm, inner sep=1.5pt, fill=blue!4},
    res/.style={draw, rounded corners, align=center, minimum height=6mm,
      inner sep=2pt, fill=red!5}]
  \node[box] (mdp) {MDP\\[0.5pt] $(\Scal,\Acal,P,R,\gamma)$};
  \node[box, right=4mm of mdp] (embed) {embed $\phi$\\[0.5pt] distortion $\eta$};
  \node[box, right=4mm of embed] (ann) {ANN index\\[0.5pt] $N(s)$, recall $r$};
  \node[box, right=4mm of ann] (backup) {two restricted runs\\[0.5pt] $f_t^{-}\!\uparrow$;\ $f_t^{+}\!\downarrow$};
  \node[box, right=4mm of backup] (fix) {sandwich\\[0.5pt] $f_t^{-}\le d\le f_t^{+}$};
  \draw[->] (mdp)--(embed);
  \draw[->] (embed)--(ann);
  \draw[->] (ann)--(backup);
  \draw[->] (backup)--(fix);
  \node[res, below=4mm of backup] (r)
    {uncovered pairs frozen: gap
     $\rho=\max\limits_{(s,s')\notin\Ccal}\big(d(s,s')-\dinit(s,s')\big)$;\
     lower arm $\norm{\dann-d}_\infty=\rho$ (Cor.~\ref{cor:identity})};
  \draw[->, dashed] (r.north) -- (backup.south);
  \end{tikzpicture}
  \caption{\textsc{Certified sub-quadratic bisimulation via ANN}. An ANN index over embedding $\phi$
    fixes covered pairs $\Ccal$; two monotone restricted iterations enclose $d$ at every sweep
    (Cor.~\ref{cor:sandwich}). Uncovered pairs freeze at gap $\rho$, which the lower arm's global
    error \emph{equals} (Cor.~\ref{cor:identity}).}
  \label{fig:overview}
  \end{figure*}
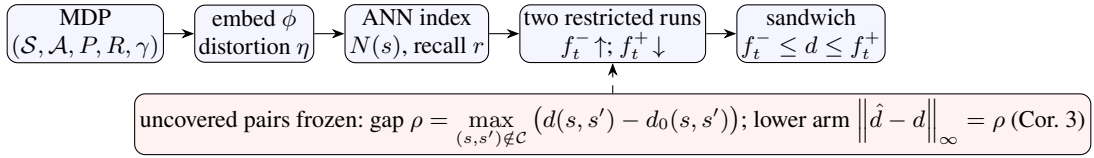

The difficulty is the error. An ANN sweep seems to cost only two local, per-step quantities, the
embedding distortion $\eta$ and recall miss $r$, amplified by $1/(1-\gamma)$ into a clean bound
$\norm{\dann-d}_\infty\le\varepsilon/(1-\gamma)$ for $\varepsilon=\eta+Lr$. \emph{Every bound of this
form is false}, not because of the constant: a top-$k$ index touches only $O(|\Scal|k)$ of the
$\binom{|\Scal|}{2}$ pairs per iteration, and without re-seeding still misses a $\Theta(|\Scal|^2)$
fraction entirely; those pairs are never updated and incur their full initialization gap as error,
invisible to per-step index quality. Proposition~\ref{prop:rho-lower} makes the refutation
unconditional, constructing two reward-oblivious instances with zero distortion and recall miss,
indistinguishable on every evaluated pair yet differing by a constant on an uncovered one. The
random-MDP grid instead verifies the corrected identity, covered-pair bound, and two-sided
enclosure under an exact Kantorovich backup.

In this paper, we replace the conjecture with a \emph{coverage-augmented} guarantee. Our contributions are as
follows.
\begin{itemize}\setlength{\itemsep}{0pt}\setlength{\parskip}{0pt}
\item We give a certificate-carrying sub-quadratic algorithm: an ANN index over a low-dimensional
embedding selects the pairs, and the exact Kantorovich operator updates those alone
(Algorithm~\ref{alg:pipeline}, Theorem~\ref{thm:complexity}).
\item We prove an anytime error bound (Theorem~\ref{thm:error}): after $t$ sweeps the global error
is at most $\max(\rho,\eop/(1-\gamma))+\gamma^t\Delta_0$, for coverage gap $\rho$, backup
perturbation $\eop$, and initialization gap $\Delta_0$; for the exact backup it sharpens to the
equality $\norm{\dann-d}_\infty=\rho$ (Corollary~\ref{cor:identity}), so the frozen block does not
bound the error, it \emph{is} the error, however good the index.
\item Since $\rho$ is unobservable, we pair that iteration with an over-estimating twin, proving a
computable anytime sandwich $f_t^{-}\le d\le f_t^{+}$ (Corollary~\ref{cor:sandwich}) whose width
bounds per-pair error and whose induced clusterings, on agreement, recover the exact covered
aggregation without computing $d$.
\item We prove matching lower bounds (Proposition~\ref{prop:rho-lower}): a quadratic obstruction
for reward-oblivious index-first selection and $\Omega(|\Scal|)$ evaluations for fully adaptive
selection, leaving the super-linear case open.
\item We verify the analysis without surrogates, using an exact Kantorovich operator, a real MDS
embedding, and a real LSH index, and we measure wall-clock scaling against the exact sweep, compare
against independently trained MICo and DBC, and run the pipeline on Taxi and a $2500$-state
gridworld.
\end{itemize}

\section{Setup and Related Work}\label{sec:setup}
We consider a finite Markov decision process $(\Scal,\Acal,P,R,\gamma)$ with states $\Scal$, actions
$\Acal$, transition kernel $P(\cdot\mid s,a)$, reward $R:\Scal\times\Acal\to[0,L]$, and discount
$\gamma\in[0,1)$ \citep{puterman1994mdp,suttonbarto2018}. Throughout, metrics are symmetric functions
on unordered pairs with zero diagonal, and $\Dg:=L/(1-\gamma)$ denotes the a-priori diameter bound
($d\le\Dg$ pointwise for every metric considered here). The bisimulation metric of
\citet{ferns2004metrics} is the fixed point $d=Td$ of the operator
\begin{equation}
  \begin{split}
  (Tf)(s,s') = \max_{a\in\Acal}\Big[\,&|R(s,a)-R(s',a)| \\
  &+ \gamma\,W_1^{f}\big(P(\cdot\mid s,a),P(\cdot\mid s',a)\big)\Big],
  \end{split}
  \label{eq:bisim}
  \end{equation}
where $W_1^{f}$ is the Kantorovich--Wasserstein-1 distance under ground cost $f$
\citep{villani2009optimal}. The operator is a $\gamma$-contraction in the sup-norm
$\norm{f}_\infty=\max_{s,s'}|f(s,s')|$, monotone in its ground cost (Lemma~\ref{lem:basic}), so $d$
exists, is unique, and is the limit of $T^n0$ \citep{comanici2015convergence}. Each application
evaluates \eqref{eq:bisim} at every one of the $\binom{|\Scal|}{2}$ pairs and solves an
optimal-transport problem per pair-action, the $\Theta(|\Scal|^2)$ cost we attack.

Probabilistic bisimulation originates with \citet{larsen1989bisimulation} and, for labelled Markov
processes, \citet{desharnais1999metrics,desharnais2004metrics}; the quantitative refinement is due
to \citet{ferns2004metrics,ferns2011bisimulation}, extended to infinite spaces
\citep{ferns2005metrics}, with domain-theoretic treatments
\citep{vanbreugel2005domain,vanbreugel2001approximation}. \citet{ferns2014bisimulation} connect the
metric to value functions, \citet{comanici2012basis} to basis-function discovery,
\citet{taylor2009bounding,ravindran2003smdp} to lax and homomorphism variants. As the fixed point
of a contraction \citep{bertsekas1996neuro,puterman1994mdp}, an iterative solver is natural; its
asynchronous convergence theory \citep{bertsekas1996neuro} is what our re-seeded pipeline
(Corollary~\ref{cor:restore}) invokes, and monotone two-sided enclosures, classical in numerical
dynamic programming, are what we adapt. Exact and on-the-fly algorithms
\citep{bacci2013computing,bacci2013onthefly} and complexity results \citep{chen2012complexity}
remain quadratic; model minimization \citep{givan2003equivalence,dean1997model} and approximate
abstraction \citep{li2006abstraction,abel2016near,jiang2015abstraction} inherit this cost when the
metric must be computed rather than assumed.

The scalable line of work sidesteps the cost by learning rather than computing:
\citet{castro2020scalable} gives sampling-based methods for deterministic MDPs, \citet{castro2021mico}
the MICo distance (kernel and continuity analyses: \citet{castro2023kernel,zhang2021metric}), and
\citet{zhang2021dbc,gelada2019deepmdp,kemertas2021robust,agarwal2021contrastive,lan2022generalization}
learn embeddings tracking a bisimulation-like distance. We are complementary: we accelerate the
metric's computation rather than replace it, using an embedding only as an index for ANN queries,
not the final representation.

The acceleration draws on sub-linear near-neighbor search via locality-sensitive hashing
\citep{indyk1998ann,datar2004lsh}, near-optimal and random-hyperplane hashing
\citep{andoni2008nearoptimal,andoni2018approximate,charikar2002simhash}, product quantization
\citep{jegou2011pq,johnson2021faiss}, graph indices \citep{malkov2020hnsw}, and surveys
\citep{wang2014hashing,li2019annsurvey,aumuller2020annbenchmarks}. The bridge from a metric to a
Euclidean index is metric embedding \citep{bourgain1985embedding,linial1995geometry}, realized by
landmark MDS \citep{cox2008mds}; the Wasserstein term is the object of computational optimal
transport \citep{villani2009optimal,peyre2019computational,cuturi2013sinkhorn}. No prior work runs
the bisimulation fixed-point iteration through an ANN index, to our knowledge; the coverage
phenomenon and its computable certificate are our theory-side contribution.

\section{Algorithms: a Certified Pipeline and Its Analyzed Core}\label{sec:algo}
Two objects organize the paper. Algorithm~\ref{alg:pipeline} is the \emph{deliverable}: a re-seeded,
two-armed pipeline maintaining a monotone lower iterate $\dann$ and upper iterate $\dup$ that
enclose the exact metric at every sweep and stop when the observable width meets tolerance.
Algorithm~\ref{alg:sqb} is its \emph{analyzed core}: a single-build restricted iteration whose
guarantees (Theorems~\ref{thm:complexity}--\ref{thm:error}, Corollary~\ref{cor:identity}) compose
across re-seed rounds into the pipeline's certificate
(Corollaries~\ref{cor:sandwich},~\ref{cor:restore}), and isolates the coverage phenomenon in its
pure form.

\begin{algorithm}[t]
\caption{Certified sub-quadratic bisimulation (pipeline)}
\label{alg:pipeline}
\begin{algorithmic}[1]
\State \textbf{input:} MDP; embedding dim $k$; neighbor budget $k'$; sweeps per round $E$;
  uniform exploration count $u\ge1$; tolerance $\mathrm{tol}$ (or threshold $\tau$)
\State $\dann\gets d_R$ (implicit lazy under-estimate); $\dup\gets\Dg$ off-diagonal;
  $\Ccal\gets\emptyset$
\Repeat
  \State $\phi\gets$ landmark-MDS embedding of the current $\dann$\label{line:reembed}
  \State build the ANN index with fresh randomness and query every state
  \State $\Ccal\gets\Ccal\cup\{\text{ANN-retrieved unordered pairs}\}$\label{line:reseed}
  \For{each $s\in\Scal$ and $j=1,\ldots,u$}
    \State draw $v\sim\mathrm{Unif}(\Scal\setminus\{s\})$ and add $\{s,v\}$ to $\Ccal$
  \EndFor
  \For{$E$ sweeps}
    \State exact restricted backup of $\dann$ and $\dup$ on every pair in $\Ccal$
      \Comment{Alg.~\ref{alg:sqb}, line~\ref{line:update}}
  \EndFor
\Until{width meets tolerance, or $\Pi^{+}_\tau=\Pi^{-}_\tau$ (Cor.~\ref{cor:sandwich}(d))}
\State \textbf{return} $\dann,\dup,\Ccal$
\end{algorithmic}
\end{algorithm}

\begin{algorithm}[t]
\caption{Single-build restricted iteration (analyzed core)}
\label{alg:sqb}
\begin{algorithmic}[1]
\State \textbf{input:} MDP; $k$; $k'$; sweeps $K$; initialization $\dinit\le d$
  (default: $d_R$, implicit) or upper freeze $\Dg$
\State $\dann\gets \dinit$ \Comment{lazy; un-updated entries read as $\dinit$}
\State $\phi\gets$ landmark-MDS embedding of $\dinit$ into $\mathbb{R}^k$\label{line:embed}
\State build ANN index over $\{\phi(s)\}$; for each $s$: $N(s)\gets k'$ retrieved neighbors
  \label{line:index}
\State $\Ccal\gets\{(s,s'): s'\in N(s)\ \text{or}\ s\in N(s')\}$
\For{$t=1,\dots,K$}
  \State $g\gets\dann$ \Comment{Jacobi update: freeze the previous sweep}
  \For{each $(s,s')\in\Ccal$}
    \State $\dann(s,s')\gets\max_a\big[\,|R(s,a)-R(s',a)|+\gamma\,W_1^{g}(P(\cdot|s,a),P(\cdot|s',a))\big]$
      \label{line:update}
  \EndFor
\EndFor
\State \textbf{return} $\dann$ and $\Ccal$
\end{algorithmic}
\end{algorithm}

The deployment pipeline mixes geometry-aware ANN retrieval with a small amount of uniform coverage
exploration. With $u=1$, an unordered pair $\{s,s'\}$ is selected in a re-seed round with
probability $p_{|\Scal|}=1-(1-\tfrac{1}{|\Scal|-1})^2\ge\tfrac{1}{|\Scal|-1}$, so by the second
Borel--Cantelli lemma every pair is selected infinitely often almost surely, converting the
eventual-coverage premise of Corollary~\ref{cor:restore} into a design guarantee; the extra
$u|\Scal|$ proposals per round are linear and leave the sub-quadratic complexity unchanged for
fixed $u$.

Both algorithms share three components, the first two run once per build: an initialization
(lower arm at an under-estimate $\dinit\le d$, either $0$ or the reward pseudo-metric
$d_R(s,s'):=\max_a|R(s,a)-R(s',a)|\le d(s,s')$; upper arm at $\Dg$; both \emph{implicit}, computed
on demand in $O(|\Acal|)$ or $O(1)$); an embedding $\phi:\Scal\to\mathbb{R}^k$ by landmark MDS
\citep{cox2008mds} (Corollary~\ref{cor:embed}) with an ANN index over $\{\phi(s)\}$
(random-hyperplane LSH \citep{charikar2002simhash,indyk1998ann}, equally a product-quantization or
graph index \citep{jegou2011pq,malkov2020hnsw}), one query per state fixing $N(s)$ and hence
$\Ccal$; and the restricted iteration itself, updating only covered pairs by the exact backup
\eqref{eq:bisim} under the current iterate.

Index quality is two measured quantities: the additive distortion
\begin{equation}
\eta \;:=\; \max_{s,s'}\big|\,\norm{\phi(s)-\phi(s')}-d(s,s')\,\big|,
\label{eq:eta}
\end{equation}
and the recall miss $r$, the fraction of a state's true $k'$ nearest embedded neighbors the index
fails to return. Under the exact backup these affect \emph{only which pairs are covered}, not
covered-value accuracy (Theorem~\ref{thm:error}(D)); they enter the error bound only for
index-side backups that read costs off the embedding (Assumption~\ref{asm:backup}, Lemma~\ref{lem:inst}).

The sweep is synchronous, matching Theorem~\ref{thm:error}'s recursion (a Gauss--Seidel variant
needs a separate asynchronous argument, not analyzed here); the returned table need not satisfy the
triangle inequality under partial coverage, so all guarantees are stated through the exact-metric
enclosure rather than by treating it as a metric in its own right. Pairs outside $\Ccal$ retain
their frozen value: a top-$k'$ index covers at most $|\Scal|k'$ of $\binom{|\Scal|}{2}$ pairs per
build, and the uncovered pairs are an error source unrelated to covered-pair accuracy, one the
two-armed enclosure makes \emph{visible}.

\section{Guarantees}\label{sec:theory}

\begin{assumption}[Bounded transition support]\label{asm:support}
$|\mathrm{supp}\,P(\cdot\mid s,a)|\le m$ for every $(s,a)$, with $m$ independent of $|\Scal|$.
\end{assumption}

Assumption~\ref{asm:support} is what makes even a \emph{single} exact backup cheap: the transport LP
at a covered pair may be solved on the union of the two supports, which is provably identical to the
dense $|\Scal|\times|\Scal|$ program (Lemma~\ref{lem:support}) and costs
$c_m=O(m^3\log m)$, independent of $|\Scal|$. Without it, one dense Kantorovich solve is already
super-linear in $|\Scal|$ and no restriction of the \emph{pair sweep} can rescue sub-quadratic time.

\begin{theorem}[Sub-quadratic complexity]\label{thm:complexity}
Suppose Assumption~\ref{asm:support} holds and the ANN index answers a $k'$-neighbor query over
$|\Scal|$ points in amortized time $q(|\Scal|)$ after an $\widetilde O(|\Scal|)$ build (LSH and
graph-index instantiations of $q$, Appendix~\ref{app:complexity}). Then Algorithm~\ref{alg:sqb}
spends $O\!\big(|\Scal|(\mathrm{poly}(k)+q(|\Scal|))\big)$ once on
lines~\ref{line:embed}--\ref{line:index}, and $O(|\Scal|\,k'\,c_m)$ per sweep thereafter; by
Theorem~\ref{thm:error}, $K=O\!\big(\log(\Delta_0/\epsilon)/(1-\gamma)\big)$ sweeps reach
covered-set accuracy $\epsilon$ (beyond the coverage floor), for total time
\[
O\!\Big(|\Scal|\big(\mathrm{poly}(k)+q(|\Scal|)\big)
\;+\;\tfrac{|\Scal|\,k'\,c_m}{1-\gamma}\log\tfrac{\Delta_0}{\epsilon}\Big),
\]
which is $\widetilde O\!\big(|\Scal|^{1+\rho_{\mathrm{LSH}}}+|\Scal|/(1-\gamma)\big)$ under LSH and
$\widetilde O\!\big(|\Scal|/(1-\gamma)\big)$ under a polylogarithmic query oracle: sub-quadratic in
$|\Scal|$ in either case, against the $\Theta(|\Scal|^2\,c_m)$ per sweep of the exact algorithm.
Algorithm~\ref{alg:pipeline} runs two arms and $R$ re-seed rounds, remaining sub-quadratic for any
$R=o\big(|\Scal|/q(|\Scal|)\big)$; driving the certificate width to zero \emph{globally} forces
$R=\Omega(|\Scal|/k')$, i.e.\ quadratic total work (Corollary~\ref{cor:restore}(d), full accounting
Appendix~\ref{app:complexity}), which is why the pipeline uses early stopping at a certified
tolerance.
\end{theorem}
\emph{Proof in Appendix~\ref{app:complexity}.}

Fix the persistent covered set $\Ccal$ of Algorithm~\ref{alg:sqb} and an initialization $\dinit\le
d$. Define the \emph{initialization gap} and the \emph{coverage gap}
\begin{equation}
\Delta_0:=\norm{\dinit-d}_\infty,
\qquad
\rho \;:=\; \max_{(s,s')\notin\Ccal}\big(d(s,s')-\dinit(s,s')\big),
\label{eq:rho}
\end{equation}
the largest initialization error over never-retrieved pairs ($\dinit=d_R$ provably shrinks $\rho$
relative to $\dinit=0$). $\rho$ is defined through the unknown metric and hence \emph{not
observable}; the theorems below characterize the error in terms of it, and
Corollary~\ref{cor:sandwich} supplies the computable surrogate. The naive conjecture
$\norm{\dann-d}_\infty\le\varepsilon/(1-\gamma)$ for a per-step index-quality term $\varepsilon$
ignores $\rho$, and Proposition~\ref{prop:rho-lower} shows it is false for \emph{every} such
$\varepsilon$.

The covered-pair update need not be exact; we allow any backup that approximates one application of
$T$ uniformly.

\begin{assumption}[$\eop$-approximate covered backup]\label{asm:backup}
There is $\eop\ge 0$ such that the covered-pair backup $B$ satisfies
$|Bf(s,s')-Tf(s,s')|\le\eop$ for every $(s,s')\in\Ccal$ and every symmetric $f$ with
$\norm{f}_\infty\le \Dg+\eop/(1-\gamma)$. The exact backup of line~\ref{line:update} satisfies this
with $\eop=0$.
\end{assumption}

An index-side backup that reads costs off the embedding instantiates $\eop$ concretely
(Lemma~\ref{lem:inst}, Appendix~\ref{app:proof}): perturbing ground-cost entries by at most $\eta$
and omitting at most an $r$-fraction of probability mass gives $\eop\le\gamma(\eta+r\Dg)$, a bound
immaterial to the refutation below since Proposition~\ref{prop:rho-lower}(i)'s construction has
$\eta=r=0$ while the error is $\rho_0$.

\begin{theorem}[Coverage-augmented approximation error, anytime form]\label{thm:error}
Let $f_0=\dinit\le d$ and $f_{t+1}=Bf_t$ on $\Ccal$, $f_{t+1}=\dinit$ off $\Ccal$, under
Assumption~\ref{asm:backup}. Write
$e_t:=\max_{(s,s')\in\Ccal}|f_t(s,s')-d(s,s')|$. Then for every $t\ge0$:
\begin{enumerate}
\item[\emph{(A)}] \emph{(covered)}
$\;e_t\;\le\;\max\!\Big(\dfrac{\eop}{1-\gamma},\ \eop+\gamma\rho\Big)\;+\;\gamma^{t}\,\Delta_0$;
\item[\emph{(B)}] \emph{(global)}
$\;\norm{f_t-d}_\infty\;\le\;\max\!\Big(\rho,\ \dfrac{\eop}{1-\gamma}\Big)\;+\;\gamma^{t}\,\Delta_0$,
and every limit point $\dann$ of $(f_t)$ satisfies
$\norm{\dann-d}_\infty\le\max\!\big(\rho,\ \eop/(1-\gamma)\big)$;
\item[\emph{(C)}] \emph{(full coverage)} if $\Ccal=\Scal\times\Scal$ then $\rho=0$ and \emph{(B)}
reads $\norm{f_t-d}_\infty\le\eop/(1-\gamma)+\gamma^t\Delta_0$;
\item[\emph{(D)}] \emph{(exact backup)} if $B$ is the exact backup of line~\ref{line:update} and
$\dinit\in\{0,d_R\}$ (any sub-solution, $B\dinit\ge\dinit$ on $\Ccal$), then $f_t$ increases
monotonically to the unique fixed point $\dann$ of the restricted iteration
(Lemma~\ref{lem:wellposed}), with $\dinit\le\dann\le d$ pointwise and
$\;\max_{\Ccal}|\dann-d|\le\gamma\rho$: index quality affects only \emph{which} pairs are covered,
never the accuracy of covered values.
\end{enumerate}
\end{theorem}
\emph{Proof in Appendix~\ref{app:proof}.} The exact-backup case admits an identity, not merely a
bound; it shows the global bound in (B) is \emph{tight}, and it is what the experiments verify to
machine precision.

\begin{corollary}[Exact coverage identity]\label{cor:identity}
Under Theorem~\ref{thm:error}(D), $\;\norm{\dann-d}_\infty=\rho$.
\end{corollary}
\emph{Proof in Appendix~\ref{app:proof}.} The identity cuts two ways: the frozen block does not merely bound the error, it \emph{is} the
error, however good the index; yet because $\rho$ is defined through the unknown $d$, it
\emph{characterizes} the error without \emph{certifying} it, which is what
Corollary~\ref{cor:sandwich}'s two-sided construction is for (why per-step bounds alone cannot
capture this: \S\ref{sec:disc}). The following lower bound shows the $\rho$ term is not loose slack
but \emph{necessary} for the index-first class Algorithm~\ref{alg:sqb} belongs to, and that
adaptivity buys at most a quadratic-to-linear reduction, not exemption.

\begin{proposition}[The coverage gap is unavoidable]\label{prop:rho-lower}
Work in the pair-evaluation model: evaluating a pair $(s,t)$ reveals the per-action reward
differences $\{|R(s,a)-R(t,a)|\}_a$ and the transported costs for that pair (all kernels identical
in the constructions below, so every transported cost is $0$ regardless of ground cost or
transport convention). Fix any $\rho_0\in(0,L]$.
\begin{enumerate}
\item[\emph{(i)}] \emph{(Oblivious coverage.)} Let the evaluated pair set $Q$ be selected without
access to the reward function (fixed in advance, or computed from transition/embedding data;
randomization allowed), with $\E|Q|=o(|\Scal|^2)$. Then there are two MDPs, indistinguishable on
every evaluated quantity, whose exact metrics differ by $\rho_0$ on an unevaluated pair; any output
computable from the evaluations errs by at least $\rho_0/2$ on one of them
($(1-o(1))\,\rho_0/2$ in expectation under randomization), regardless of how unevaluated pairs are
assigned.
\item[\emph{(ii)}] \emph{(Adaptive coverage.)} There is a family of MDPs with
$|\Acal|=\lfloor|\Scal|/2\rfloor$ actions on which \emph{every} algorithm (adaptive and
randomized) that performs fewer than $\lfloor|\Scal|/2\rfloor$ pair evaluations suffers
$\E\,\norm{\dann-d}_\infty\ge\rho_0/2$.
\end{enumerate}
Consequently the $\rho$ term of Theorem~\ref{thm:error}(B) cannot be removed for oblivious
coverage, and Corollary~\ref{cor:identity} shows it is achieved with equality; adaptive schemes need
$\Omega(|\Scal|)$ evaluations even on this family, with the known adaptive route to exactness
degenerating to quadratic work through dependency closure (\S\ref{sec:disc}); a super-linear
adaptive lower bound remains open. Since the two instances of (i) are indistinguishable, no
computable certificate can be tight on unevaluated pairs: the certificate of
Corollary~\ref{cor:sandwich} is exact in its \emph{validity} while its \emph{width} there honestly
reports the full prior interval.
\end{proposition}
\emph{Proof in Appendix~\ref{app:rho-lower}.} We verify construction (i) directly (Appendix~\ref{app:rho-lower}, Table~\ref{tab:rho-lower}): the
realized global error equals the planted gap exactly (correlation $1.0$), and at fixed budget
$k'=8$ it stays pinned at $\rho_0$ as coverage falls from $42\%$ at $|\Scal|=20$ to $5\%$ at
$|\Scal|=160$, while full coverage drives it to $0$.

\subsection{Two-Sided Certificates and Downstream Aggregation}\label{sec:downstream}

Since $\rho$ is unobservable (Theorem~\ref{thm:error}(B), Corollary~\ref{cor:identity}), is the
metric still useful for state abstraction \citep{ferns2004metrics,li2006abstraction} (merge states
below $\tau$, solve the aggregated MDP)? Far pairs driving $\rho$ should never be merged, but this
holds only under an \emph{upper-bound} treatment of un-retrieved pairs; under the lower arm's
under-estimate, a frozen pair reads as near-identical instead. The upper arm repairs this; pairing
both yields the computable certificate.

\begin{proposition}[The upper arm: over-estimation and conservative aggregation]\label{prop:downstream}
Run the restricted iteration with the exact backup, initializing \emph{all} entries at the constant
upper bound $U:=\Dg=L/(1-\gamma)$ and freezing un-retrieved pairs at $U$ ($U$ computable in $O(1)$,
$U\ge d$ pointwise, no quadratic diameter computation needed). Write $f_t^{+}$ for the iterates,
$\dup$ for the limit, and $\rhoup:=\max_{(s,s')\notin\Ccal}(U-d(s,s'))$. Then:
\begin{enumerate}
\item[\emph{(i)}] the iterates decrease monotonically and every iterate, hence $\dup$, is a
pointwise \emph{over}-estimate: $f_t^{+}\ge\dup\ge d$ everywhere;
\item[\emph{(ii)}] for any $\tau<U$, every directly linked pair ($\dup(s,s')\le\tau$, necessarily
covered) has $d(s,s')\le\tau$, so the single-linkage clustering of $\dup$ at $\tau$ \emph{refines}
the exact metric's: every $\dup$-cluster lies in a $d$-cluster;
\item[\emph{(iii)}] so any value-loss certificate monotone in within-cluster diameter
\citep{li2006abstraction,ferns2014bisimulation} certifies a loss for $\dup$-aggregation no larger
than for exact-metric aggregation at $\tau$;
\item[\emph{(iv)}] \emph{(mirror of Theorem~\ref{thm:error}(D) and Corollary~\ref{cor:identity})}
at the fixed point, $\max_{\Ccal}(\dup-d)\le\gamma\rhoup$ and
$\norm{\dup-d}_\infty=\rhoup$.
\end{enumerate}
\end{proposition}
\emph{Proof in Appendix~\ref{app:downstream}.}

The two arms are complementary, the lower under-estimating with error $\rho$ and the upper
over-estimating with $\rhoup$, neither observable alone; run together, they observe each other.

\begin{corollary}[Anytime two-sided certificate]\label{cor:sandwich}
Run the lower arm ($f_t^{-}$: exact backup, initialization $\dinit\in\{0,d_R\}$,
Theorem~\ref{thm:error}(D)) and upper arm ($f_t^{+}$: Proposition~\ref{prop:downstream}) on the same
covered sets, including under Corollary~\ref{cor:restore}'s re-seeded schedule where $\Ccal$ grows.
Then:
\begin{enumerate}
\item[\emph{(a)}] \emph{(anytime sandwich)} $f_t^{-}\le d\le f_t^{+}$ pointwise for every $t$, with
$f_t^{-}$ nondecreasing and $f_t^{+}$ nonincreasing, so $[f_t^{-}(p),f_t^{+}(p)]$ is a valid,
shrinking interval for $d(p)$ at every sweep and pair;
\item[\emph{(b)}] \emph{(computable certificate)} the width $w_t:=f_t^{+}-f_t^{-}\ge0$ is computable
from the two runs, with $|f_t^{\pm}(p)-d(p)|\le w_t(p)$ for every $p$; on uncovered pairs
$w_t(p)=U-\dinit(p)$ exactly, so the coverage frontier is visible in the certificate itself;
\item[\emph{(c)}] \emph{(limit width on covered pairs)} at the fixed points,
$w_\infty(p)\le\gamma(\rho+\rhoup)$ for every $p\in\Ccal$, by
Theorem~\ref{thm:error}(D) and Proposition~\ref{prop:downstream}(iv);
\item[\emph{(d)}] \emph{(certified aggregation)} for $\tau<U$, the single-linkage clusterings
$\Pi^{+}_\tau,\Pi^{-}_\tau$ of $f_t^{+},f_t^{-}$ over covered pairs bracket $\Pi^{\Ccal}_\tau$ (the
exact metric's covered clustering): $\Pi^{+}_\tau$ refines $\Pi^{\Ccal}_\tau$, which refines
$\Pi^{-}_\tau$; if $\Pi^{+}_\tau=\Pi^{-}_\tau$, both equal $\Pi^{\Ccal}_\tau$, recovering it
\emph{with certainty} without computing $d$. Under re-seeding, $\Ccal\to\Scal\times\Scal$ and
$\Pi^{\Ccal}_\tau$ becomes the full exact clustering.
\end{enumerate}
\end{corollary}
\emph{Proof in Appendix~\ref{app:downstream}.} This turns Corollary~\ref{cor:identity}'s identity
into an operational tool: Algorithm~\ref{alg:pipeline} re-seeds until the width meets tolerance or
$\Pi^{+}_\tau=\Pi^{-}_\tau$ (Corollary~\ref{cor:restore} guarantees both eventually), at the price
of a factor of two in sweep cost and an honesty on uncovered pairs Proposition~\ref{prop:rho-lower}
shows cannot be tightened.

\begin{corollary}[Anytime refinement via re-seeding]\label{cor:restore}
Let Algorithm~\ref{alg:pipeline} rebuild the index with fresh randomness, re-embed the lower
iterate every $E$ sweeps, retain cumulative coverage, run both exact-backup arms, with any fixed
exploration count $u\ge1$. Then: (a) lower iterates remain monotone nondecreasing and bounded by
$d$, upper iterates monotone nonincreasing and bounded below by $d$, so Corollary~\ref{cor:sandwich}(a)'s
sandwich holds at every step; (b) every pair is retrieved infinitely often almost surely, so both
limits equal the exact metric by asynchronous fixed-point convergence for monotone sup-norm
contractions \citep{bertsekas1996neuro}; (c) at every finite horizon, Theorem~\ref{thm:error},
Corollary~\ref{cor:identity}, and Proposition~\ref{prop:downstream}(iv) apply to the cumulative set
$\Ccal_t$ and its nonincreasing gaps $\rho_t,\rhoup_t$; and (d) zero global width still requires
$\binom{|\Scal|}{2}$ cumulative retrievals, hence $\Omega(|\Scal|^2)$ total work.

Re-seeding ablations at $|\Scal|=18$ confirm part (b): pure-LSH alone saturates below full coverage
(near-neighbor bias misses the far pairs setting $\rho$), while adding one uniform partner per
state reaches full coverage and $\Pi^{+}_\tau=\Pi^{-}_\tau$ within ten rounds
(Appendix~\ref{app:reseed}).
\end{corollary}

\section{Experiments}\label{sec:exp}

The experiments validate Theorem~\ref{thm:error}, Corollaries~\ref{cor:identity}
and~\ref{cor:sandwich}, and Proposition~\ref{prop:rho-lower} where every named quantity is exactly
measurable, and exhibit the naive bound's coverage failure. Bound-validation instances are small
and CPU-only of necessity, since the exact Wasserstein LP per pair-action is an optimal-transport
problem: random MDPs with $|\Scal|\in\{10,18\}$, $|\Acal|=3$, $\gamma=0.7$, rewards rescaled so
$L=1$, $10$ seeds per configuration, $95\%$ confidence intervals; a separate timing study reaches
$|\Scal|$ up to $1600$ (cheap cost) and $400$ (full $W_1$), and Taxi and a $2500$-state gridworld
exercise the full pipeline. Appendix~\ref{app:repro} details infrastructure, seeding, and the
number of runs behind every reported result.

Each component is implemented exactly as stated, with no surrogate: a \emph{true} Wasserstein-1
backup by Kantorovich LP under the current iterate ($\eop=0$, so Theorem~\ref{thm:error}(D),
Corollary~\ref{cor:identity}, Proposition~\ref{prop:downstream}(iv) apply), a real MDS embedding
with distortion $\eta$ measured pairwise, and a real random-hyperplane LSH index with recall miss
$r$ measured against true nearest neighbors (full protocol, Appendix~\ref{app:protocol}). Across
the eight-configuration grid, mean distortion is $\eta=0.107$ and coverage gap averages
$\rho=1.16$ under $\dinit=0$, shrinking to $0.73$ (ratio $0.63$) under $\dinit=d_R$, as predicted.

The per-step-only conjecture is already refuted unconditionally by Proposition~\ref{prop:rho-lower}(i)
with $\eta=r=0$; Table~\ref{tab:full} audits the conventional quantity $(\eta+Lr)/(1-\gamma)$ on
eight configurations ($L=1$, $\gamma=0.7$). Two rows ($|\Scal|=10,k'=6$) violate it (naive
$0.803,1.061$ against realized error $1.10,1.14$, since $13$ and $10$ of $45$ pairs are never
updated); across the full $80$-run grid it fails in $11/80$ ($13.75\%$), while the corrected bound,
the exact identity $\norm{\dann-d}_\infty=\rho$, and the sandwich enclosure hold in \emph{every}
run, the identity to machine precision, covered error never exceeding $\gamma\rho$ (largest row
mean $0.33$). Figures~\ref{fig:coverage}--\ref{fig:bound} (Appendix~\ref{app:boundaudit}) plot
these values from the table below.

\begin{table*}[t]\centering\small
\setlength{\tabcolsep}{3.3pt}
\begin{tabular}{rrrrrrrrrrrr}
\toprule
$|\Scal|$ & $k'$ & pl. & $\eta$ & $r$ & $\rho$ & cov. err & glob. err & never & work & naive & corrected \\
\midrule
$10$ & $3$ & $2$ & $.0047$ & $.5000$ & $1.16$ & $.33$ & $1.16$ & $25/45$ & $.30$ & $1.682$ & $1.682$ \\
$10$ & $3$ & $6$ & $.0162$ & $.3333$ & $1.17$ & $.33$ & $1.17$ & $28/45$ & $.30$ & $1.165$ & $1.170$ \\
$10$ & $6$ & $2$ & $.0241$ & $.2167$ & $1.10$ & $.27$ & $1.10$ & $13/45$ & $.60$ & $.803$ & $1.100$ \\
$10$ & $6$ & $6$ & $.0184$ & $.3000$ & $1.14$ & $.29$ & $1.14$ & $10/45$ & $.60$ & $1.061$ & $1.140$ \\
$18$ & $3$ & $2$ & $.1983$ & $.4815$ & $1.18$ & $.30$ & $1.18$ & $121/153$ & $.17$ & $2.266$ & $2.266$ \\
$18$ & $3$ & $6$ & $.1992$ & $.4259$ & $1.18$ & $.30$ & $1.18$ & $118/153$ & $.17$ & $2.084$ & $2.084$ \\
$18$ & $6$ & $2$ & $.1974$ & $.5370$ & $1.17$ & $.31$ & $1.17$ & $80/153$ & $.33$ & $2.448$ & $2.448$ \\
$18$ & $6$ & $6$ & $.1983$ & $.4815$ & $1.17$ & $.31$ & $1.17$ & $91/153$ & $.33$ & $2.266$ & $2.266$ \\
\bottomrule
\end{tabular}
\caption{Synchronized exact-operator audit ($10$ seeds per configuration, $80$ runs). ``pl.'' is
the number of LSH hyperplanes; ``never'' counts uncovered unordered pairs; ``work'' is $k'/|\Scal|$;
``naive'' is $(\eta+Lr)/(1-\gamma)$ with $L=1$ and $\gamma=0.7$; and ``corrected'' is
$\max\{\rho,\text{naive}\}$. The two $|\Scal|=10,k'=6$ rows violate the naive expression but
satisfy the corrected bound. Global error equals $\rho$ in every row, and covered error remains
below $\gamma\rho$. Figures~\ref{fig:coverage} and~\ref{fig:bound} use this exact row order.}
\label{tab:full}
\end{table*}

The restricted sweep's work ratio $k'/|\Scal|$ ranges $0.17$--$0.60$ (Table~\ref{tab:full}),
matching Theorem~\ref{thm:complexity}; on these small instances the exact LP dominates runtime, so
operation count, not wall clock, is relevant. A dedicated timing study isolates pair-visit scaling
directly (Table~\ref{tab:wallclock}): cheap ground cost across $|\Scal|\in\{50,\dots,1600\}$ gives
log-log exponents $1.91$ exact versus $1.41$ ANN, speedup $1.2\times\to7.2\times$
(Figure~\ref{fig:wallclock}, Appendix~\ref{app:complexity}); the full exact-Wasserstein operator
across $|\Scal|\in\{50,\dots,400\}$ (union-support LP, Lemma~\ref{lem:support}, gap $0.0$) gives
exponents $2.02$ versus $1.00$, speedup $3.1\times\to25.5\times$, covered-pair error bounded by
$\gamma\rho$ throughout (realized $\le0.006$), as Corollary~\ref{cor:identity} predicts.

\begin{table}[t]\centering\small
\begin{tabular}{llrrrr}
\toprule
setting & $|\Scal|$ & exact (s) & ANN (s) & speedup & cov.\ err\\
\midrule
\multirow{4}{*}{cheap cost}
 & $50$   & 0.0012 & 0.0010 & $1.2\times$ & --\\
 & $200$  & 0.0146 & 0.0040 & $3.7\times$ & --\\
 & $800$  & 0.2103 & 0.0384 & $5.5\times$ & --\\
 & $1600$ & 0.8616 & 0.1192 & $7.2\times$ & --\\
\midrule
\multirow{4}{*}{full $W_1$}
 & $50$  & 13.05  & 4.24  & $3.1\times$ & 0.004\\
 & $100$ & 53.07  & 8.59  & $6.2\times$ & 0.005\\
 & $200$ & 212.64 & 17.06 & $12.5\times$ & 0.006\\
 & $400$ & 871.83 & 34.20 & $25.5\times$ & 0.006\\
\bottomrule
\end{tabular}
\caption{Wall-clock per-sweep cost, exact all-pairs versus ANN top-$k'$. ``cheap cost'' uses a fixed
shared ground cost so the pair-visit count alone drives runtime; ``full $W_1$'' solves the exact
Kantorovich transport LP per pair-action on the union of supports, identical to the dense LP by
Lemma~\ref{lem:support} (verified numerically, gap $0.0$), with $k'=8$. Covered-pair error is
bounded by $\gamma\rho$ throughout (Theorem~\ref{thm:error}(D)).}
\label{tab:wallclock}
\end{table}

On the same eight configurations, the two-arm sandwich holds ($f_t^{-}\le d\le f_t^{+}$ in $80/80$
runs), both one-sided identities hold to machine precision, and $\Pi^{+}_\tau=\Pi^{-}_\tau$ is rare
at a single build ($9/240$, $3.7\%$) but reached under re-seeding in two of three seeds within $10$
rounds (Appendix~\ref{app:reseed}). The full pipeline (re-seeding every $E=3$ sweeps, $32$
configurations) matches or beats the one-shot variant in $32/32$ runs (mean error $0.22$ vs
$1.18$), tracking $\max(\rho_t,e_t)$ exactly as Corollary~\ref{cor:identity} predicts
(Appendix~\ref{app:reseed}). A budget-matched random-coverage baseline gives a nearly identical
coverage gap to ANN ($\rho$ ratio $1.04$), confirming $\rho$ is a property of the budget, not the
index; with the upper arm, ANN's near-neighbor bias still cuts mean-pair error $2.5\times$ over
random ($0.54\to0.21$). At larger scale ($|\Scal|\in\{50,100,200\}$, Appendix~\ref{app:scale}) the
identity again holds in $30/30$ runs, while the naive bound, never violated there, is merely
\emph{vacuous}: with fixed embedding dimension, $\eta$ grows with $|\Scal|$, pushing the naive
expression past the metric diameter. At $|\Scal|=120$ (Appendix~\ref{app:scale}), the upper-arm ANN
metric matches exact-metric aggregation statistically (loss $0.022\pm0.008$ against
$0.022\pm0.008$, mean $12.8$ clusters against $12$), while under-estimate and reward-only baselines
collapse to one cluster (loss $1.47\pm0.55$).

On Gymnasium Taxi-v3 ($|\Scal|=500$, deterministic, $\gamma=0.9$), the post-hoc lower-arm identity
gives $\rho=180.4$, close to the exact diameter $200.5$, unavailable at run time; the observable
sandwich stays wide (coverage only $5\%$--$20\%$), so the conservative upper arm certifies no
unsupported merges and the pipeline abstains, while an uncertified diagnostic on the raw embedding
collapses to one cluster (loss $45.9$ against $15.8$ exact; exact computation takes $0.7$s). At a
scale where the exact sweep is genuinely prohibitive ($|\Scal|=2500$,
$\binom{2500}{2}\approx3.12$M pairs, $25$ planted rooms, learned $\mathbb{R}^8$ embedding), $8$
upper-freeze sweeps at $k'=20$ ($12.8\%$ of one sweep, $\approx1200$s) recover $254$--$284$
over-fragmented clusters with mean loss $2.27$ against a reward-only baseline's $3.18$, a $28.6\%$
improvement (full detail, Appendix~\ref{app:taxigrid}).

\paragraph{Comparison to learned surrogates (MICo, DBC).}
MICo \citep{castro2021mico} and DBC \citep{zhang2021dbc} are trained independently with their
published objectives on grouped MDPs ($|\Scal|=64$, eight planted groups, $|\Acal|=3$,
$\gamma=0.9$); our method uses their encoder only as an ANN index, then applies the exact restricted
backup, so the comparison is independently trained surrogates versus exact-operator refinement,
not three methods sharing one distance.

Sweeping retrieval budget over four seeds per point (Table~\ref{tab:surrogates},
Appendix~\ref{app:surrogates}): at $k'=8$ ($25\%$ coverage) our loss is $0.056$, reaching the
exact-metric skyline ($0.013$) once coverage passes half ($k'\ge16$), while MICo and DBC stay
$22$--$33\times$ worse ($0.292$, $0.426$) at every budget, and the two-arm interval encloses the
exact metric in all $16$ runs, which neither surrogate provides.

The claim is benchmark-scoped: a larger check at $|\Scal|\in\{60,120\}$ shows the same coverage
percentage need not give the same loss across state counts and geometries (loss $1.08$ against a
$0.026$ skyline at $|\Scal|=120$, below the coverage that reaches it), and a finer sweep at fixed
$|\Scal|=60$ confirms the same half-of-pairs crossover ($53.5\%$) while MICo and DBC stay flat,
never enclosing the exact metric (full detail, Appendix~\ref{app:surrogates}).

\section{Discussion}\label{sec:disc}
\paragraph{Why the naive bound is so tempting, and so wrong.}
The clean bound $\varepsilon/(1-\gamma)$ mechanically combines per-step perturbations with the
contraction, correct exactly where the contraction acts (Theorem~\ref{thm:error}'s
$\eop/(1-\gamma)$ branch, collapsing to $\gamma\rho$ on covered pairs for the exact backup). It
forgets the operator never touches uncovered pairs: unlike the usual approximation-error story
where every coordinate updates imperfectly, here a $\Theta(|\Scal|^2)$ block is frozen, and
Corollary~\ref{cor:identity} shows that block does not merely bound the error, it \emph{is} the
error. Any scheme restricting a global fixed point to a sparse retrieved set must account for
un-retrieved coordinates separately from retrieved quality; a monotone two-sided enclosure
(Corollary~\ref{cor:sandwich}) is the generic repair wherever the operator is monotone with known
sub- and super-solutions.

\paragraph{What the guarantee does and does not say.}
Theorem~\ref{thm:error} bounds iterate and metric error against the exact metric, not downstream
value or policy error, though Proposition~\ref{prop:downstream} composes it with standard
abstraction certificates \citep{ferns2014bisimulation,li2006abstraction}. $\rho$ is
\emph{characterized} exactly (Corollary~\ref{cor:identity}) but \emph{not observable}, so the
pipeline exposes the sandwich width instead; deriving $\rho$ from index parameters under a
distributional embedding assumption remains open. The complexity result is a worst-case count under
an explicit query oracle $q(\cdot)$, inheriting the chosen index's recall-speed trade-off,
unoptimized here. Proposition~\ref{prop:rho-lower}'s adaptive side is deliberately modest:
$\Omega(|\Scal|)$ evaluations are necessary even adaptively, the oblivious $\Theta(|\Scal|^2)$
barrier is real for the index-first class, and the fully adaptive super-linear question is open.

\paragraph{The choice of index and the recall--coverage distinction.}
Recall $r$ is \emph{per-state}, driven down by a better index (more hyperplanes, a graph structure
\citep{malkov2020hnsw}, product quantization \citep{jegou2011pq}); coverage is \emph{global and
cumulative}: a top-$k'$ scheme proposes at most $k'$ per state, so even perfect recall covers only
$O(|\Scal|k')$ pairs, leaving a quadratic remainder frozen. Under the exact backup the separation
is total: Theorem~\ref{thm:error}(D) says index quality affects only coverage, so only the
retrieval policy, not the index, can fix the error floor (Corollary~\ref{cor:restore}). Our
verification fixes the index family (LSH) and varies granularity: recall miss ranges $0.54$ to
$0.22$ while global error stays pinned at $\rho\approx1.16$.

\paragraph{Relation to learned surrogates.}
Scalable methods \citep{castro2020scalable,castro2021mico,zhang2021dbc} replace the metric with a
learned distance, accepting it may not be the true one; we keep the exact operator and approximate
only the sweep, controlling error relative to the true metric. The two are complementary: a
learned embedding of the kind those methods produce is precisely the index we need, and feeding it
into Algorithm~\ref{alg:pipeline} recovers a metric certifiably (Corollary~\ref{cor:sandwich})
related to the true one, not merely empirically useful; as Taxi shows, the certificate also flags
uninformative embeddings at run time. The kernel perspectives of
\citet{castro2023kernel,zhang2021metric,lan2022generalization} suggest such embeddings are often
low-distortion in Corollary~\ref{cor:embed}'s sense, so cheap-embedding regimes are ones those
methods already occupy.

\paragraph{On-demand exactness degenerates to quadratic.}
The on-the-fly exact method of \citet{bacci2013computing,bacci2013onthefly} computes the metric on
a query set $Q$ via dependency closure (BFS from $Q$ through transition supports to convergence),
the natural \emph{adaptive} competitor of Proposition~\ref{prop:rho-lower}(ii). On sparse-support
MDPs ($|\Scal|\in\{50,100\}$, support $4$, $k'=12$) the closure saturates to $100\%$ of all pairs
from a few hundred queries, since random supports mix every pair in within a few BFS steps
($425$--$452$s versus $13$--$17$s for ANN, $25$--$33\times$). The trade is exactness for time:
on-demand is exact on $Q$, the ANN lower arm exactly $\rho$ (Corollary~\ref{cor:identity}); once a
transition graph mixes, as any connected tabular MDP does, the closure propagates everywhere and
on-demand exactness degenerates to the full quadratic sweep. Algorithm~\ref{alg:pipeline} escapes
this by \emph{not} chasing exactness on unvisited pairs, a price
Proposition~\ref{prop:rho-lower}(i)-(ii) shows unavoidable for index-first and (linearly) adaptive
selection alike, made visible by the certificate.

\section{Conclusion}\label{sec:conc}
Bisimulation metrics admit sub-quadratic, certificate-carrying approximation: a low-dimensional
index selects pairs, and the exact restricted operator runs from both sides. Coverage, not index
quality, bounds global error: skipped pairs set the anytime limit $\max(\rho,\eop/(1-\gamma))$, an
equality under exact backups, and the sandwich width certifies per-pair uncertainty and the covered
clustering. Lower bounds are quadratic for reward-oblivious index-first coverage,
$\Omega(|\Scal|)$ for fully adaptive evaluation (super-linear open). Experiments confirm identity
and enclosure in every seed, sub-quadratic wall-clock scaling (to $25.5\times$), MICo and DBC
$22$--$33\times$ above the skyline we reach at half coverage, abstention on Taxi's weak embedding,
and $28.6\%$ lower value loss at $|\Scal|=2500$ for $12.8\%$ of one sweep. 

\bibliography{references}

\clearpage
\newpage

\appendix
\setcounter{secnumdepth}{1}
\section*{Technical Appendix (Supplementary Material)}

\section{Complexity Details for Theorem~\ref{thm:complexity}}\label{app:complexity}

\paragraph{One-time costs.}
(i) \emph{Embedding.} Landmark MDS with $m_{\mathrm{lm}}=O(\mathrm{poly}(k))$ landmarks reads the
$m_{\mathrm{lm}}$ landmark columns of $\dinit$; each entry of $\dinit\in\{0,d_R\}$ is computable on
demand in $O(|\Acal|)$, so forming the landmark distances costs
$O(|\Scal|\,m_{\mathrm{lm}}\,|\Acal|)$ and the eigen-step and out-of-sample projections cost
$O(m_{\mathrm{lm}}^3+|\Scal|\,m_{\mathrm{lm}}k)$ \citep{cox2008mds}, i.e.\
$O(|\Scal|\,\mathrm{poly}(k))$ overall for fixed $|\Acal|$. Classical MDS from the full matrix is
$\Theta(|\Scal|^2)$ and is never invoked. (ii) \emph{Index build:} $\widetilde O(|\Scal|)$ for LSH
and standard graph indices \citep{charikar2002simhash,malkov2020hnsw,li2019annsurvey}.
(iii) \emph{Queries:} $|\Scal|$ queries at amortized $q(|\Scal|)$; for LSH,
$q(|\Scal|)=O(|\Scal|^{\rho_{\mathrm{LSH}}})$ with $\rho_{\mathrm{LSH}}<1$ set by the approximation
ratio \citep{indyk1998ann,andoni2008nearoptimal}; graph indices give polylogarithmic $q$
empirically but without worst-case guarantees \citep{malkov2020hnsw,aumuller2020annbenchmarks},
which is why the theorem is stated relative to an explicit oracle.

\paragraph{Per-sweep cost.}
$|\Ccal|\le|\Scal|k'$ exact backups, each solving $|\Acal|$ transport problems on the union of two
supports of size at most $m$ (Assumption~\ref{asm:support}), identical in value to the dense
program by Lemma~\ref{lem:support}; a network-simplex or specialized solver runs in
$c_m=O(m^3\log m)$ \citep{peyre2019computational}, independent of $|\Scal|$. The two-armed pipeline
doubles this constant.

\paragraph{Sweep count.}
By Theorem~\ref{thm:error}, the transient term is $\gamma^t\Delta_0$, so covered-set accuracy
$\epsilon$ beyond the coverage floor is reached at
$t\ge\log(\Delta_0/\epsilon)/\log(1/\gamma)$, and $\log(1/\gamma)\ge1-\gamma$ gives
$K=O\!\big(\log(\Delta_0/\epsilon)/(1-\gamma)\big)$; the upper arm obeys the same recursion with
$\Delta_0^{+}:=\norm{U-d}_\infty\le\Dg$.

\paragraph{Totals.}
Summing, the total time is
$O\!\big(|\Scal|(\mathrm{poly}(k)+q(|\Scal|))\big)+O\!\big(|\Scal|k'c_m
\log(\Delta_0/\epsilon)/(1-\gamma)\big)$: under LSH,
$\widetilde O(|\Scal|^{1+\rho_{\mathrm{LSH}}}+|\Scal|/(1-\gamma))$; under a polylogarithmic
oracle, $\widetilde O(|\Scal|/(1-\gamma))$. Both are sub-quadratic; the exact sweep costs
$\Theta(|\Scal|^2 c_m)$ per application. The re-seeded pipeline adds the one-time block once per
round: for $R$ rounds the total is
$\widetilde O\!\big(R\,|\Scal|\,q(|\Scal|)+R\,E\,|\Scal|k'c_m\big)$, sub-quadratic for
$R=o(|\Scal|/q(|\Scal|))$, while zero global width forces $R=\Omega(|\Scal|/k')$
(Corollary~\ref{cor:restore}(d)) and hence quadratic total work, consistent with
Proposition~\ref{prop:rho-lower}.

\begin{figure}[t]\centering
\includegraphics[width=\linewidth]{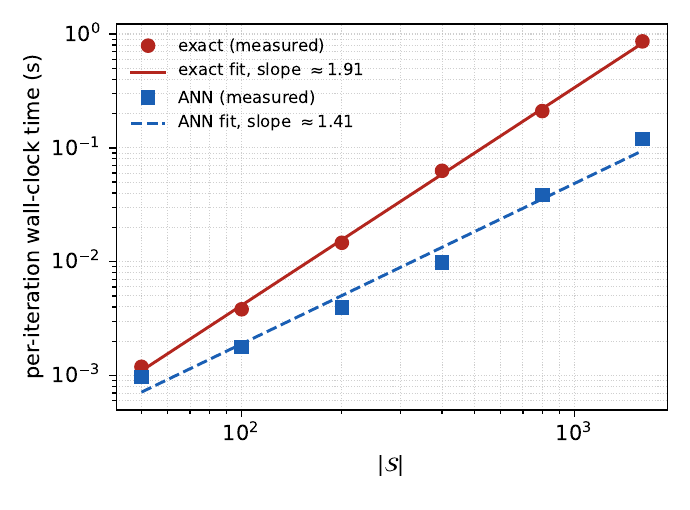}
\caption{Measured per-sweep wall-clock time against $|\Scal|$ on log--log axes, for the cheap
ground-cost rows of Table~\ref{tab:wallclock}. Fitted slopes are $1.91$ for the exact all-pairs
sweep and $1.41$ for the ANN top-$k'$ sweep, the quadratic versus sub-quadratic separation of
Theorem~\ref{thm:complexity}. Both sweeps share the identical inner kernel and differ only in which
pairs they visit, so the slope gap measures pair-visit scaling alone.}
\label{fig:wallclock}
\end{figure}

\begin{corollary}[Embedding does not dominate]\label{cor:embed}
With $O(\mathrm{poly}(k))$ landmarks, landmark MDS \citep{cox2008mds} reads
$O(|\Scal|\,\mathrm{poly}(k))$ entries of $\dinit$ (each computable on demand in $O(|\Acal|)$) and
costs $O(|\Scal|\,\mathrm{poly}(k))$ time, once; classical MDS from the full distance matrix would
cost $\Theta(|\Scal|^2)$ and is never used. When the bisimulation metric has low intrinsic
dimension, so a small $k$ achieves distortion $\eta=o(1)$
\citep{bourgain1985embedding,linial1995geometry}, the one-time embedding cost is dominated by the
query-and-backup cost of Theorem~\ref{thm:complexity}.
\end{corollary}

\section{Proof of Theorem~\ref{thm:error} and Corollary~\ref{cor:identity}}\label{app:proof}

\paragraph{Setup.}
Let $\mathcal{P}$ be the unordered off-diagonal pairs of $\Scal$ and $\mathcal{B}(\mathcal{P})$ the
bounded symmetric functions on $\mathcal{P}$ (extended by $0$ on the diagonal) with
$\norm{f}_\infty=\max_{p\in\mathcal{P}}|f(p)|$. The operator $T$ is \eqref{eq:bisim}; recall
$\Dg=L/(1-\gamma)$ and $\norm{d}_\infty\le\Dg$ (from $d=\lim_n T^n0$ and
$\norm{T^n0}_\infty\le L\sum_{i<n}\gamma^i$).

\begin{lemma}[Basic properties of $T$]\label{lem:basic}
For all $f,g\in\mathcal{B}(\mathcal{P})$ and all distributions $\mu,\nu$ on $\Scal$:
(a) $|W_1^{f}(\mu,\nu)-W_1^{g}(\mu,\nu)|\le\norm{f-g}_\infty$;
(b) if $f\le g$ entrywise then $W_1^{f}(\mu,\nu)\le W_1^{g}(\mu,\nu)$;
(c) $\norm{Tf-Tg}_\infty\le\gamma\norm{f-g}_\infty$;
(d) $f\le g$ implies $Tf\le Tg$;
(e) $d_R\le Tf\le L+\gamma\norm{f}_\infty$ whenever $f\ge0$, where
$d_R(s,s')=\max_a|R(s,a)-R(s',a)|$;
(f) if $f\ge g-c$ entrywise for some $c\ge0$, then $Tf\ge Tg-\gamma c$.
\end{lemma}
\begin{proof}
(a) Let $\pi$ be an optimal coupling for $W_1^{g}$. Its cost under $f$ exceeds its cost under $g$ by
at most $\norm{f-g}_\infty$ because $\pi$ has unit mass; hence
$W_1^{f}\le W_1^{g}+\norm{f-g}_\infty$, and symmetrically. (b) Every coupling costs no more under
$f$ than under $g$; take the infimum. (c) The reward term is cost-independent, $W_1^{\cdot}$ moves
by at most $\norm{f-g}_\infty$ by (a), and $\max_a$ is nonexpansive. (d) By (b) and monotonicity of
$\max_a$. (e) $W_1^{f}\ge0$ gives the lower bound; $W_1^{f}\le\norm{f}_\infty$ (unit mass, entries
bounded) gives the upper. (f) Apply (b) with $g-c\le f$, then (a)-style shift:
$W_1^{f}\ge W_1^{g-c}\ge W_1^{g}-c$, and the reward term is unchanged.
\end{proof}

\begin{lemma}[Union-support transport]\label{lem:support}
For distributions $\mu,\nu$ with supports $A,B\subseteq\Scal$ and any ground cost $c$, the
Kantorovich LP restricted to $A\times B$ has the same value as the LP over
$\Scal\times\Scal$.
\end{lemma}
\begin{proof}
Every coupling of $(\mu,\nu)$ places mass only on $A\times B$, so the two feasible sets coincide
after deleting identically-zero variables.
\end{proof}

\paragraph{The restricted iteration and its boundedness.}
Fix the persistent covered set $\Ccal$, the initialization $\dinit\le d$, and a covered backup $B$
satisfying Assumption~\ref{asm:backup}. The iterates are $f_0=\dinit$ and
\[
f_{t+1}(p)=
\begin{cases}
Bf_t(p), & p\in\Ccal,\\
\dinit(p), & p\notin\Ccal.
\end{cases}
\]
By Lemma~\ref{lem:basic}(e) and $|Bf-Tf|\le\eop$,
$\norm{f_{t+1}}_\infty\le\max\big(\norm{\dinit}_\infty,\,L+\gamma\norm{f_t}_\infty+\eop\big)$; the
ball $\norm{f}_\infty\le(L+\eop)/(1-\gamma)=\Dg+\eop/(1-\gamma)$ is invariant and contains
$f_0$ (as $0\le\dinit\le d\le\Dg$), so Assumption~\ref{asm:backup} applies at every iterate.

\begin{lemma}[Well-posedness of the exact restricted iteration]\label{lem:wellposed}
Let $\mathcal{A}:=\{f\in\mathcal{B}(\mathcal{P}):f=\dinit\text{ off }\Ccal\}$ and let
$\hat T$ act as $T$ on $\Ccal$ and as the identity off $\Ccal$. Then $\hat T$ maps $\mathcal{A}$ to
itself and is a $\gamma$-contraction on $(\mathcal{A},\norm{\cdot}_\infty)$; it has a unique fixed
point $\dann\in\mathcal{A}$, and the exact iterates converge to it geometrically.
\end{lemma}
\begin{proof}
For $f,g\in\mathcal{A}$ the difference vanishes off $\Ccal$, and on $\Ccal$,
$|\hat Tf-\hat Tg|=|Tf-Tg|\le\gamma\norm{f-g}_\infty$ by Lemma~\ref{lem:basic}(c); note the
sup on the right runs over \emph{all} pairs, which is exactly why the frozen boundary must agree,
as it does within $\mathcal{A}$. Banach's theorem on the closed set $\mathcal{A}$ concludes.
\end{proof}

\paragraph{Proof of Theorem~\ref{thm:error}.}
\emph{Uncovered pairs.} For every $t$ and $p\notin\Ccal$, $f_t(p)=\dinit(p)$, so
$|f_t(p)-d(p)|=d(p)-\dinit(p)\in[0,\rho]$ with maximum exactly $\rho$ by definition
\eqref{eq:rho}.

\emph{Covered recursion.} Write $e_t=\max_{p\in\Ccal}|f_t(p)-d(p)|$. For $p\in\Ccal$,
\[
\begin{split}
|f_{t+1}(p)-d(p)|&=|Bf_t(p)-Td(p)|\\
&\le\underbrace{|Bf_t(p)-Tf_t(p)|}_{\le\,\eop}
+\underbrace{|Tf_t(p)-Td(p)|}_{\le\,\gamma\norm{f_t-d}_\infty},
\end{split}
\]
using Assumption~\ref{asm:backup} (legitimate by the boundedness invariant) and
Lemma~\ref{lem:basic}(c). Since $\norm{f_t-d}_\infty=\max(e_t,\max_{p\notin\Ccal}(d-\dinit)(p))
\le\max(e_t,\rho)$, we obtain $e_{t+1}\le\psi(e_t)$ with
$\psi(x):=\eop+\gamma\max(x,\rho)$.

\emph{Solving the recursion.} $\psi$ is nondecreasing and $\gamma$-Lipschitz on
$\mathbb{R}_{\ge0}$, hence has a unique fixed point $x^\star$. Case analysis: if
$x\ge\rho$, the fixed-point equation reads $x=\eop+\gamma x$, i.e.\ $x=\eop/(1-\gamma)$,
consistent iff $\eop/(1-\gamma)\ge\rho$; if $x\le\rho$ it reads $x=\eop+\gamma\rho$, consistent iff
$\eop+\gamma\rho\le\rho$, i.e.\ $\eop\le(1-\gamma)\rho$. The two consistency conditions are
complementary, and in both regimes
$x^\star=\max\big(\eop/(1-\gamma),\ \eop+\gamma\rho\big)$: when $\eop\ge(1-\gamma)\rho$,
$\eop/(1-\gamma)-(\eop+\gamma\rho)=\gamma\big(\eop/(1-\gamma)-\rho\big)\ge0$; otherwise the
inequality reverses. By monotonicity, $e_t\le\psi^t(e_0)$, and by the Lipschitz bound
$|\psi^t(e_0)-x^\star|\le\gamma^t|e_0-x^\star|$, so
$e_t\le x^\star+\gamma^t e_0\le x^\star+\gamma^t\Delta_0$ (as
$e_0=\max_{\Ccal}|\dinit-d|\le\Delta_0$). This is (A).

\emph{Global bound.} $\norm{f_t-d}_\infty=\max(e_t,\rho)\le\max(x^\star,\rho)+\gamma^t\Delta_0$.
If $\eop\le(1-\gamma)\rho$ then $\eop+\gamma\rho\le\rho$ and $\eop/(1-\gamma)\le\rho$, so
$\max(x^\star,\rho)=\rho$; otherwise $\eop+\gamma\rho<\eop/(1-\gamma)$ and
$\max(x^\star,\rho)=\eop/(1-\gamma)$. In both cases
$\max(x^\star,\rho)=\max\big(\rho,\ \eop/(1-\gamma)\big)$, giving (B); the limit-point statement
follows by letting $t\to\infty$ along a convergent subsequence. If $\Ccal=\Scal\times\Scal$ the
uncovered maximum is over the empty set, $\rho=0$ by convention, and the recursion is
$e_{t+1}\le\eop+\gamma e_t$, giving (C).

\emph{Exact backup (D).} Here $B=T$ on $\Ccal$, so the iteration is $\hat T$ of
Lemma~\ref{lem:wellposed}. $\hat T$ is monotone (Lemma~\ref{lem:basic}(d) on $\Ccal$; identity off
$\Ccal$). Sub-solution: for $\dinit\in\{0,d_R\}$, Lemma~\ref{lem:basic}(e) gives
$T\dinit\ge d_R\ge\dinit$ on $\Ccal$, so $f_1\ge f_0$, and monotonicity propagates
$f_{t+1}\ge f_t$ for all $t$. Upper invariance: if $f\le d$ then $\hat Tf\le\hat Td\le d$
(on $\Ccal$, $Tf\le Td=d$; off $\Ccal$, $\dinit\le d$), so $f_t\le d$ throughout. A monotone
bounded sequence on a finite pair set converges pointwise; the limit lies in $\mathcal{A}$, is a
fixed point of $\hat T$ by continuity, and equals the unique $\dann$ of
Lemma~\ref{lem:wellposed}, with $\dinit\le\dann\le d$. Finally, for $p\in\Ccal$,
\[
\begin{split}
d(p)-\dann(p)&=Td(p)-T\dann(p)\le\gamma\norm{d-\dann}_\infty\\
&=\gamma\max\big(e_\infty,\ \max_{q\notin\Ccal}(d-\dinit)(q)\big)\\
&\le\gamma\max(e_\infty,\rho).
\end{split}
\]
If the inner maximum were $e_\infty$ we would get $e_\infty\le\gamma e_\infty$, i.e.\
$e_\infty=0\le\gamma\rho$; otherwise $e_\infty\le\gamma\rho$ directly. Either way
$e_\infty\le\gamma\rho$, which is (D). \hfill$\qed$

\paragraph{Proof of Corollary~\ref{cor:identity}.}
Under (D), $\dann=\dinit$ off $\Ccal$, so the uncovered error attains its maximum $\rho$ at the
argmax pair of \eqref{eq:rho}; the covered error is at most $\gamma\rho\le\rho$. Hence
$\norm{\dann-d}_\infty=\max(e_\infty,\rho)=\rho$ (and $=0$ when $\Ccal$ covers all pairs).
\hfill$\qed$

\begin{lemma}[Index-side instantiation]\label{lem:inst}
Suppose the covered backup evaluates \eqref{eq:bisim} with each ground-cost entry perturbed by at
most $\eta$ (e.g.\ read off the embedding) and with a coupling that omits at most an $r$-fraction of
probability mass, re-routed at per-unit cost error at most $\Dg$. Then Assumption~\ref{asm:backup}
holds with $\eop\le\gamma(\eta+r\Dg)$.
\end{lemma}

\paragraph{Proof of Lemma~\ref{lem:inst}.}
Fix $f$ in the invariant ball and a covered pair. Perturbing each ground-cost entry by at most
$\eta$ moves the transport value by at most $\eta$ (the argument of Lemma~\ref{lem:basic}(a) applied
to the perturbed cost). Omitting an $r$-fraction of mass and re-routing it at per-unit cost error at
most $\Dg$ moves the value by at most $r\Dg$. The reward term is exact and $\max_a$ is nonexpansive,
so the backup differs from $Tf$ by at most $\gamma(\eta+r\Dg)$, i.e.\
Assumption~\ref{asm:backup} holds with $\eop\le\gamma(\eta+r\Dg)$. \hfill$\qed$

\paragraph{Why the naive bound fails.}
Any per-step bound $\varepsilon/(1-\gamma)$ silently assumes $\Ccal=\Scal\times\Scal$, i.e.\ that
the contraction acts on every coordinate. A top-$k'$ index covers $O(|\Scal|k')$ of the
$\binom{|\Scal|}{2}$ pairs; without re-seeding the complement is $\Theta(|\Scal|^2)$ and frozen, so
$\rho>0$ generically and, by Corollary~\ref{cor:identity}, the realized error is exactly $\rho$,
however small $\varepsilon$ is. Proposition~\ref{prop:rho-lower}(i) makes this unconditional with a
construction on which every per-step quantity is zero. The verification of \S\ref{sec:exp} exhibits
the failure on generic instances under the conventional accounting
(e.g.\ at $|\Scal|=10$, $k'=6$, coverage gap $\rho\approx1.10$ against per-step term
$\approx0.80$).

\section{Proof of Proposition~\ref{prop:rho-lower}}\label{app:rho-lower}

Throughout, all transition kernels are identical across states and instances (say
$P(\cdot\mid s,a)=\mathrm{Unif}(\Scal)$), so for any ground cost the transported cost between any
two states' next-state distributions is $0$; the pair-evaluation answers are therefore the
per-action reward differences together with zeros, whatever transport convention the model adopts,
and the exact metric is $d(s,t)=\max_a|R(s,a)-R(t,a)|$. All constructions place rewards in
$[0,\rho_0]\subseteq[0,L]$ after adding the constant $\rho_0/2$, which changes no difference.

\paragraph{Part (i): oblivious coverage.}
Given an unordered pair $(i,j)$ and a sign $b\in\{+1,-1\}$, define the instance $M_{i,j,b}$ by
$R(x,a)=0$ for $x\notin\{i,j\}$, $R(i,a)=\rho_0/2$, and $R(j,a)=b\,\rho_0/2$, for every $a$. Then
$d(i,j)=\rho_0\cdot\mathbf{1}[b=-1]$, while for every other pair the answers are
$b$-independent: $|R(i,a)-R(x,a)|=|R(j,a)-R(x,a)|=\rho_0/2$ and $|R(x,a)-R(y,a)|=0$ for
$x,y\notin\{i,j\}$.

Let the coverage rule select $Q$ without access to $R$ (obliviousness); its distribution is
therefore independent of $(i,j,b)$. Deterministic case: $|Q|=o(|\Scal|^2)<\binom{|\Scal|}{2}$, so
some pair $(i,j)\notin Q$; plant there. Every evaluated quantity is identical in $M_{i,j,+1}$ and
$M_{i,j,-1}$, so any output computable from the evaluations (including any embedding built from
them and any value interpolated in that embedding) takes the same value $v$ for $(i,j)$ in both
instances, and $\max_b|v-d_b(i,j)|\ge\big(|v-0|+|v-\rho_0|\big)/2\ge\rho_0/2$. Randomized case:
draw $(i,j)$ uniformly from all pairs and $b$ uniformly, independently of the rule's randomness;
$\Pr[(i,j)\in Q]\le\E|Q|/\binom{|\Scal|}{2}=o(1)$, and conditioned on $(i,j)\notin Q$ the transcript
is $b$-independent, so
$\E\norm{\dann-d}_\infty\ge\E\big[\,|v_{(i,j)}-d_b(i,j)|\;\big|\;(i,j)\notin Q\big]
\Pr[(i,j)\notin Q]\ge(1-o(1))\,\rho_0/2$.

\emph{Why obliviousness is needed.} A reward-adaptive prober defeats any single planted pair:
evaluating $(1,x)$ for all $x$ returns $\rho_0/2$ exactly when $x\in\{i,j\}$ (or reveals
$1\in\{i,j\}$), so $O(|\Scal|)$ evaluations locate the pair, and one more reads it. Part (ii)
therefore hides $\Theta(|\Scal|)$ independent bits, each readable only at its own pair.

\paragraph{Part (ii): adaptive coverage.}
Let $|\Scal|=n$ be even, group states into partner pairs $P_\ell=\{2\ell-1,2\ell\}$ for
$\ell\le n/2$, and take $|\Acal|=n/2$ actions $a_1,\dots,a_{n/2}$. Draw independent uniform bits
$s_\ell\in\{+1,-1\}$ and set $R(2\ell-1,a_\ell)=\rho_0/2$, $R(2\ell,a_\ell)=s_\ell\rho_0/2$, and
$R(x,a_m)=0$ otherwise. The evaluation of a partner pair $P_\ell$ returns, in coordinate $a_\ell$,
$\tfrac{\rho_0}{2}|1-s_\ell|\in\{0,\rho_0\}$, revealing $s_\ell$; hence
$d(P_\ell)\in\{0,\rho_0\}$ according to $s_\ell$. The evaluation of any non-partner pair
$\{u,v\}$ returns, in each coordinate $a_m$, the value $\rho_0/2$ if exactly one of $u,v$ lies in
$P_m$ and $0$ otherwise, independent of every bit, since $|s_m\rho_0/2-0|=\rho_0/2$ regardless
of $s_m$.

Consequently, after any sequence of evaluations, the transcript is a deterministic function of the
queries and of the bits of the partner pairs evaluated so far; by the principle of deferred
decisions, conditioned on the transcript $\tau$, the bits of unevaluated partner pairs remain
independent and uniform. Suppose the algorithm (adaptive; randomized, by additionally conditioning
on its random string) performs $Q<n/2$ evaluations. At most $Q$ partner pairs are evaluated, so the
transcript-measurable set of unevaluated partner pairs is nonempty; let $p^\star(\tau)$ be the one
of smallest index. The output value $v_{p^\star}$ is $\tau$-measurable, while
$d(p^\star)\in\{0,\rho_0\}$ is uniform given $\tau$, so
$\E\big[\norm{\dann-d}_\infty\,\big|\,\tau\big]\ge
\E\big[\,|v_{p^\star}-d(p^\star)|\,\big|\,\tau\big]\ge\rho_0/2$; taking expectation over $\tau$
gives the claim. \hfill$\qed$

\begin{remark}
The construction of (ii) uses $|\Acal|=\Theta(|\Scal|)$ actions; with $|\Acal|$ constant it hides
$|\Acal|$ bits and yields an $\Omega(\min(|\Acal|,|\Scal|/2))$ adaptive bound. Evaluating all
$\binom{|\Scal|}{2}$ pairs reveals $d$ exactly on this family, consistent with
Theorem~\ref{thm:error}(C); whether adaptive algorithms require $\omega(|\Scal|)$ evaluations on
some family is open. The indistinguishability in (i) also proves the certificate-tightness remark
in Proposition~\ref{prop:rho-lower}: any computable interval that were narrower than
$[0,\rho_0]$ at the planted pair would exclude the true value in one of the two instances, so the
sandwich's full-width report on uncovered pairs is not conservatism but necessity.
\end{remark}

\begin{table}[t]
\centering\small
\begin{tabular}{lrrrr}
\toprule
planted gap $\rho_0$ & $0.1$ & $0.4$ & $0.8$ & $1.2$ \\
\midrule
realized global error & $0.1$ & $0.4$ & $0.8$ & $1.2$ \\
\bottomrule
\end{tabular}
\caption{Matching lower bound, construction (i) of Proposition~\ref{prop:rho-lower}
(reward-oblivious coverage; correlation between planted gap and realized error is $1.0$). The
realized global error of a bounded-coverage oblivious algorithm equals the planted hidden-pair
distance $\rho_0$ exactly; the hidden pair uses identical kernels, so embedding distortion and
recall are irrelevant by construction. At fixed budget $k'=8$ the error is independent of $|\Scal|$
as coverage $\to0$ (falling from $42\%$ at $|\Scal|=20$ to $5\%$ at $|\Scal|=160$), and full
coverage drives it to $0$.}
\label{tab:rho-lower}
\end{table}

\section{The Upper Arm and the Sandwich: Proofs and Verification}\label{app:downstream}

\paragraph{Mechanism.}
Under the lower arm's under-estimate freeze, an un-retrieved pair reads as its initialization
(for $\dinit=0$, as \emph{identical}), so a threshold rule merges states that were never compared;
worse, the $W_1$ backup at a covered pair transports against a cost matrix whose un-retrieved
entries are under-estimates, dragging covered values down as well (Theorem~\ref{thm:error}(D)
bounds, but does not eliminate, this one-sided effect). Both pressures push distinct states into
one cluster. Freezing un-retrieved pairs at an upper bound reverses the sign of every error, which
is what makes aggregation safe and the two-armed enclosure possible.

\begin{lemma}[Constant super-solution]\label{lem:super}
Let $U\equiv\Dg=L/(1-\gamma)$ off-diagonal ($0$ on the diagonal). Then $U\ge d$ pointwise and
$TU\le U$ on every pair.
\end{lemma}
\begin{proof}
$d\le\Dg$ pointwise (Appendix~\ref{app:proof}, setup). For the second claim,
$W_1^{U}(\mu,\nu)\le\norm{U}_\infty=\Dg$ for any $\mu,\nu$ (unit mass), so
$TU\le L+\gamma\Dg=\Dg=U$.
\end{proof}

\paragraph{Proof of Proposition~\ref{prop:downstream}.}
\emph{(i).} Initialize $f_0^{+}=U$ everywhere and freeze $f_t^{+}=U$ off $\Ccal$. By
Lemma~\ref{lem:super}, $f_1^{+}=\hat Tf_0^{+}\le f_0^{+}$ (on $\Ccal$, $TU\le U$; off $\Ccal$,
equality), and monotonicity of $\hat T$ (Lemma~\ref{lem:basic}(d)) propagates
$f_{t+1}^{+}\le f_t^{+}$. Lower invariance: $f_0^{+}=U\ge d$, and if $f\ge d$ then on $\Ccal$,
$Tf\ge Td=d$, while off $\Ccal$, $U\ge d$; hence $f_t^{+}\ge d$ for all $t$. A monotone bounded
sequence on finitely many pairs converges to $\dup\ge d$, the unique fixed point of $\hat T$ on
$\{f=U\text{ off }\Ccal\}$ (the argument of Lemma~\ref{lem:wellposed} verbatim).

\emph{(ii).} If $\dup(s,s')\le\tau<U$ then $(s,s')\in\Ccal$ (uncovered entries equal $U$) and
$d(s,s')\le\dup(s,s')\le\tau$ by (i). For refinement, let $s\sim_{\dup}s'$ at threshold
$\tau$, i.e.\ there is a chain $s=u_0,u_1,\dots,u_k=s'$ with $\dup(u_{i-1},u_i)\le\tau$ for
each $i$; every link satisfies $d(u_{i-1},u_i)\le\tau$, so the same chain witnesses
$s\sim_{d}s'$. Hence each single-linkage cluster of $\dup$ is contained in a cluster of $d$ at
the same $\tau$.

\emph{(iii).} If cluster $K$ of $\dup$ is contained in cluster $K'$ of $d$, then
$\max_{s,s'\in K}d(s,s')\le\max_{s,s'\in K'}d(s,s')$, so the maximum within-cluster exact-metric
diameter under $\dup$-aggregation is at most that under $d$-aggregation at the same $\tau$.
The optimal value function is $1$-Lipschitz in $d$ \citep{ferns2014bisimulation}, so value
variation within any cluster is at most its $d$-diameter, and the approximate-abstraction
value-loss bounds of \citet{li2006abstraction} (for a fixed within-cluster weighting) are
nondecreasing functions of the within-cluster discrepancy. Any such certificate therefore assigns
$\dup$-aggregation a loss bound no larger than exact-metric aggregation at the same
threshold. This is a comparison of \emph{certificates}, which is the object a practitioner can
compute; it implies nothing weaker than the exact-metric guarantee at the same $\tau$.

\emph{(iv).} Uncovered pairs sit at $U$ for every $t$, so their error is $U-d\in[0,\rhoup]$ with
maximum exactly $\rhoup$ by definition. On covered pairs, at the fixed point and using
$\dup\ge d$,
\[
\begin{split}
\dup(p)-d(p)=T\dup(p)-Td(p)&\le\gamma\norm{\dup-d}_\infty\\
&=\gamma\max\big(e^{+}_\infty,\ \rhoup\big),
\end{split}
\]
where $e^{+}_\infty:=\max_{\Ccal}(\dup-d)$; if the inner maximum were $e^{+}_\infty$ we would get
$e^{+}_\infty\le\gamma e^{+}_\infty=0\le\gamma\rhoup$, otherwise $e^{+}_\infty\le\gamma\rhoup$
directly. Hence $e^{+}_\infty\le\gamma\rhoup\le\rhoup$ and
$\norm{\dup-d}_\infty=\max(e^{+}_\infty,\rhoup)=\rhoup$, the mirror of
Theorem~\ref{thm:error}(D) and Corollary~\ref{cor:identity}. \hfill$\qed$

\paragraph{Proof of Corollary~\ref{cor:sandwich}.}
\emph{(a) with growing coverage.} Let $\Ccal_0\subseteq\Ccal_1\subseteq\cdots$ be the cumulative
covered sets of the pipeline and consider the lower arm (the upper arm is the mirror image with
all inequalities reversed and $U$ in place of $\dinit$). We prove by induction that
$f_{t+1}^{-}\ge f_t^{-}$ and $f_t^{-}\le d$ pointwise for all $t$. Both hold at $t=0$. For the
upper invariant: if $f\le d$ then every backup value $Tf(p)\le Td(p)=d(p)$
(Lemma~\ref{lem:basic}(d)) and every frozen value is $\dinit(p)\le d(p)$, so $f_{t+1}^{-}\le d$.
For monotonicity, take any pair $p$. If $p\notin\Ccal_{t+1}$ then
$f_{t+1}^{-}(p)=\dinit(p)=f_t^{-}(p)$. If $p\in\Ccal_t$ then
$f_{t+1}^{-}(p)=Tf_t^{-}(p)\ge Tf_{t-1}^{-}(p)=f_t^{-}(p)$ by the induction hypothesis and
Lemma~\ref{lem:basic}(d). If $p$ is newly covered ($p\in\Ccal_{t+1}\setminus\Ccal_t$), then
$f_t^{-}(p)=\dinit(p)$ and $f_{t+1}^{-}(p)=Tf_t^{-}(p)\ge d_R(p)\ge\dinit(p)$ by
Lemma~\ref{lem:basic}(e): the newly covered entry leaves its frozen value in the monotone
direction. Hence $f_t^{-}$ is nondecreasing and $\le d$; the upper arm is nonincreasing and
$\ge d$ (using $TU\le U$, Lemma~\ref{lem:super}, at newly covered entries), giving the sandwich
$f_t^{-}\le d\le f_t^{+}$ at every $t$.

\emph{(b).} Immediate from (a): $d(p)\in[f_t^{-}(p),f_t^{+}(p)]$, so both endpoint errors are at
most the width; off the current covered set both arms sit at their initializations, so
$w_t(p)=U-\dinit(p)$ there.

\emph{(c).} At the fixed points on a common final covered set,
$w_\infty(p)=(\dup-d)(p)+(d-\dann)(p)\le\gamma\rhoup+\gamma\rho$ for $p\in\Ccal$, by
Proposition~\ref{prop:downstream}(iv) and Theorem~\ref{thm:error}(D).

\emph{(d).} Over covered pairs, (a) gives the edge-set inclusions
$\{p\in\Ccal: f_t^{+}(p)\le\tau\}\subseteq\{p\in\Ccal: d(p)\le\tau\}
\subseteq\{p\in\Ccal: f_t^{-}(p)\le\tau\}$. Connected components coarsen as edges are added, so
$\Pi^{+}_\tau$ refines $\Pi^{\Ccal}_\tau$, which refines $\Pi^{-}_\tau$. If the two ends of a
refinement chain coincide, every intermediate partition coincides with them; hence
$\Pi^{+}_\tau=\Pi^{-}_\tau$ forces all three equal, certifying $\Pi^{\Ccal}_\tau$ exactly. Under
re-seeding with every pair eventually covered, $\Pi^{\Ccal}_\tau$ is eventually the clustering of
the exact metric over all pairs. \hfill$\qed$

\begin{remark}[Single-linkage chains]\label{rem:downstream}
A cluster can contain pairs at true distance up to (chain length)$\times\tau$ under \emph{both}
metrics, so ``no merged pair exceeds $\tau$'' is false for any single-linkage scheme, exact or
approximate; the correct comparative statements are the refinement in
Proposition~\ref{prop:downstream}(ii) and the squeeze in Corollary~\ref{cor:sandwich}(d).
\end{remark}

\begin{remark}[Approximate backups and margins]\label{rem:margin}
If the covered backup is only $\eop$-approximate (Assumption~\ref{asm:backup}), the enclosure
degrades one-sidedly on each arm: by Lemma~\ref{lem:basic}(f), if $f\ge d-c$ then
$Bf\ge Tf-\eop\ge d-(\gamma c+\eop)$, so the invariant $f_t^{+}\ge d-c_t$ holds with
$c_{t+1}=\gamma c_t+\eop$, $c_0=0$, hence $c_t\uparrow\eop/(1-\gamma)$, and symmetrically
$f_t^{-}\le d+\eop/(1-\gamma)$. The corrected enclosure is therefore
$f_t^{-}-\eop/(1-\gamma)\le d\le f_t^{+}+\eop/(1-\gamma)$, and merging only when
$f_t^{+}\le\tau-\eop/(1-\gamma)$ restores $d\le\tau$ on every direct link, so
Proposition~\ref{prop:downstream}(ii)--(iii) and Corollary~\ref{cor:sandwich}(d) go through with
the shrunken threshold and widened certificate.
\end{remark}

\paragraph{Verification.}
On MDPs with planted behavioral groups (within-group distance $O(\varepsilon)$, cross-group
$\Theta(1)$), we aggregate at thresholds between the two bands, solve the aggregated MDP, lift the
greedy policy, and measure normalized value loss against $V^\star$ (Table~\ref{tab:downstream}). In the operating regime
(thresholds at which exact-metric aggregation is itself faithful), upper-arm aggregation matches
exact-metric aggregation to within $0.0$ mean absolute value loss over eight seeds, with a positive
over-estimation bias ($+0.33$) on covered cross-group pairs exactly as
Proposition~\ref{prop:downstream}(i) predicts, while the under-estimate freeze collapses every
instance to a single cluster (value loss $2.86$ versus $0.03$ for exact). Past the operating
regime, where the exact metric's own single-linkage begins to over-merge, the upper arm is no worse
than exact, consistent with the refinement of (ii). The sandwich checks on the same instances:
enclosure at every sweep in all $80$ runs, and the refinement chain
$\Pi^{+}_\tau\preceq\Pi^{\Ccal}_\tau\preceq\Pi^{-}_\tau$ in all $80$; the certifying equality
$\Pi^{+}_\tau=\Pi^{-}_\tau$ is rare at the single-build budget ($3.7\%$ of threshold--configuration
triples) because generic random MDPs do not exhibit a clean distance margin at arbitrary thresholds.
On the planted-group benchmark, the threshold lies between separated within- and cross-group bands,
so the same squeeze is much easier and holds in $8/8$ seeds at $|\Scal|=120$. Re-seeding with a
uniform-exploration component drives generic coverage to $100\%$ and closes the remaining ambiguity
in two of three seeds within $10$ rounds (Corollary~\ref{cor:sandwich}(d)). The coverage gap is thus benign for the downstream task
precisely when the metric is maintained as an upper bound, and jointly with the lower arm it is
\emph{certifiably} benign wherever the two arms agree.

\begin{table}[t]
\centering\small
\setlength{\tabcolsep}{3pt}
\resizebox{\linewidth}{!}{%
\begin{tabular}{lrrrr}
\toprule
threshold (within$\to$cross) & $0.25$ & $0.5$ & $0.75$ & $0.9$ \\
\midrule
exact-metric aggregation & $0.03$ & $0.03$ & $0.67$ & $1.63$ \\
ANN, upper arm ($U$) & $0.03$ & $0.03$ & $0.03$ & $0.94$ \\
ANN, lower arm (under-est.\ freeze) & $2.86$ & $2.86$ & $2.86$ & $2.86$ \\
\bottomrule
\end{tabular}%
}
\caption{Downstream value loss of the lifted greedy policy (normalized, $\downarrow$ better; eight
seeds, $|\Scal|=30$, six planted groups). In the operating regime (bands $0.25$--$0.5$) upper-arm
aggregation matches exact-metric aggregation, as Proposition~\ref{prop:downstream} certifies,
whereas the under-estimate freeze over-merges to one cluster. Beyond the regime the upper arm is no
worse than exact, consistent with refinement.}
\label{tab:downstream}
\end{table}

\section{Experimental Protocol and Implementation}\label{app:protocol}

Each component the theory refers to is implemented exactly as the algorithms state, with no
surrogate. The bisimulation backup \eqref{eq:bisim} uses a \emph{true} Wasserstein-1 term solved by
Kantorovich linear program under ground cost equal to the current iterate (the exact backup of
line~\ref{line:update}; $\eop=0$, so Theorem~\ref{thm:error}(D), Corollary~\ref{cor:identity}, and
Proposition~\ref{prop:downstream}(iv) apply); the embedding $\phi$ is a real multidimensional-scaling
embedding into $\mathbb{R}^6$ \citep{cox2008mds}, with the additive distortion $\eta$ of
\eqref{eq:eta} \emph{measured} pairwise; and the index is a real random-hyperplane LSH index
\citep{charikar2002simhash} over $\phi$, built \emph{once per run} (persistent coverage), with the
recall miss $r$ \emph{measured} against the true embedding nearest neighbors. No product-coupling
surrogate, fixed-embedding shortcut, or exact-sort stand-in for the ANN index is used. The
bound-validation arm initializes at $\dinit=0$ (worst case, for which Corollary~\ref{cor:identity}
predicts the largest identity value); the deployment arm at $\dinit=d_R$ (implicit). Across the
synchronized eight-configuration grid, the row-mean distortions average $\eta=0.107$ and the
coverage gap averages $\rho=1.16$ (under $\dinit=0$); with $\dinit=d_R$ the gap shrinks to $0.73$
(ratio $0.63$), as the reward pseudo-metric provably reduces it.

\section{The Bound Audit: Full Figures}\label{app:boundaudit}

Figures~\ref{fig:coverage} and~\ref{fig:bound} are generated directly from Table~\ref{tab:full} in
the main text and use its exact row order.

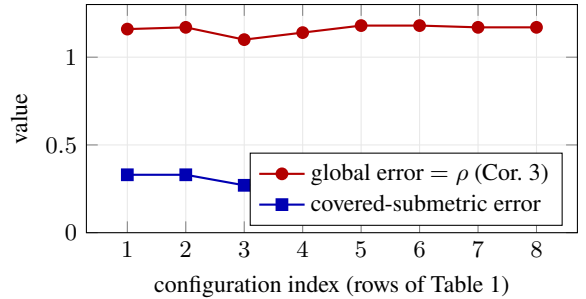
\begin{figure}[t]\centering
\begin{tikzpicture}
\begin{axis}[width=0.96\linewidth,height=4.6cm,
  xlabel={configuration index (rows of Table~\ref{tab:full})},
  ylabel={value},xtick={1,2,3,4,5,6,7,8},ymin=0,ymax=1.3,font=\small,
  legend pos=south east,legend cell align=left,grid=both,grid style={gray!18}]
\addplot[mark=*,thick,red!70!black] coordinates {(1,1.16)(2,1.17)(3,1.10)(4,1.14)(5,1.18)(6,1.18)(7,1.17)(8,1.17)};
\addlegendentry{global error $=\rho$ (Cor.~\ref{cor:identity})}
\addplot[mark=square*,thick,blue!70!black] coordinates {(1,0.33)(2,0.33)(3,0.27)(4,0.29)(5,0.30)(6,0.30)(7,0.31)(8,0.31)};
\addlegendentry{covered-submetric error}
\end{axis}
\end{tikzpicture}
\caption{Realized error on covered pairs (blue) versus globally (red), in the row order of
Table~\ref{tab:full}. The covered error stays below the exact-backup bound $\gamma\rho$ of
Theorem~\ref{thm:error}(D), while the global error equals the coverage gap $\rho$ set by the
never-retrieved pairs (Corollary~\ref{cor:identity}): the index updates its neighbors accurately but
cannot reach the rest.}
\label{fig:coverage}
\end{figure}

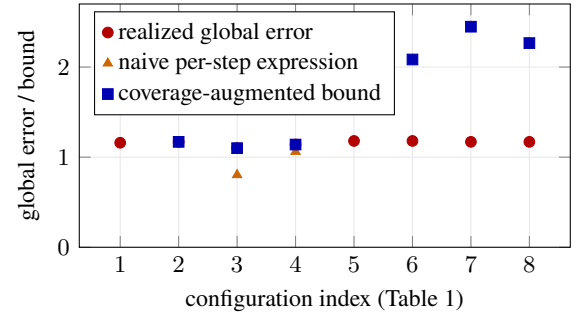
\begin{figure}[t]\centering
\begin{tikzpicture}
\begin{axis}[width=0.96\linewidth,height=4.8cm,
  xlabel={configuration index (Table~\ref{tab:full})},ylabel={global error / bound},
  xtick={1,2,3,4,5,6,7,8},ymin=0,ymax=2.7,font=\small,
  legend pos=north west,legend cell align=left,grid=both,grid style={gray!18}]
\addplot[only marks,mark=*,red!70!black] coordinates
  {(1,1.16)(2,1.17)(3,1.10)(4,1.14)(5,1.18)(6,1.18)(7,1.17)(8,1.17)};
\addlegendentry{realized global error}
\addplot[only marks,mark=triangle*,orange!80!black] coordinates
  {(1,1.682)(2,1.165)(3,0.803)(4,1.061)(5,2.266)(6,2.084)(7,2.448)(8,2.266)};
\addlegendentry{naive per-step expression}
\addplot[only marks,mark=square*,blue!70!black] coordinates
  {(1,1.682)(2,1.170)(3,1.100)(4,1.140)(5,2.266)(6,2.084)(7,2.448)(8,2.266)};
\addlegendentry{coverage-augmented bound}
\end{axis}
\end{tikzpicture}
\caption{Synchronized bound audit in the row order of Table~\ref{tab:full}. The naive expression
falls below realized error in configurations $3$ and $4$. The corrected series is computed as
$\max\{\rho,(\eta+Lr)/(1-\gamma)\}$ from the same table and is therefore valid in every row;
configuration $2$ illustrates why the maximum must be evaluated rather than rounded independently.}
\label{fig:bound}
\end{figure}

\section{Certificate Verification and Re-Seeding}\label{app:reseed}

\paragraph{The sandwich certificate, verified.}
On the eight configurations of Table~\ref{tab:full} ($10$ seeds, $80$ runs) we run the upper arm
alongside the lower arm and check Corollary~\ref{cor:sandwich}: the enclosure $f_t^{-}\le d\le
f_t^{+}$ holds at every sweep in $80/80$ runs; the lower and upper identities
$\norm{f^{-}-d}_\infty=\rho$ and $\norm{\dup-d}_\infty=\rhoup$ (Corollary~\ref{cor:identity},
Proposition~\ref{prop:downstream}(iv)) each hold to machine precision in $80/80$; the limit covered
width never exceeds $\gamma(\rho+\rhoup)$ ($80/80$; realized mean covered width $0.80$ against the
mean bound $2.70$); and on uncovered pairs the width equals $U-\dinit$ exactly, i.e.\ the certificate
itself displays the coverage frontier. The refinement chain
$\Pi^{+}_\tau\preceq\Pi^{\Ccal}_\tau\preceq\Pi^{-}_\tau$ of Corollary~\ref{cor:sandwich}(d) holds in
$80/80$, but the \emph{equality} that certifies the exact covered clustering is rare at the
single-build budget ($\Pi^{+}_\tau=\Pi^{-}_\tau$ in only $9$ of $240$ threshold-configuration-seed
cases, $3.7\%$): a single index leaves too many covered links undecided between the arms. Re-seeding
with fresh hyperplanes is the intended remedy, but pure LSH re-seeding saturates at
$\approx82$--$84\%$ coverage (its near-neighbor bias never draws the far pairs that set $\rho$), so
equality is reached only on the covered portion (one of three seeds). Adding a single
uniform-random pair per state per round, which makes every pair retrieved infinitely often (the
hypothesis of Corollary~\ref{cor:restore}), drives coverage to $100\%$ in all three seeds and
reaches full-clustering equality $\Pi^{+}_\tau=\Pi^{-}_\tau$ in two of three within $10$ rounds. The
certificate's honesty on uncovered pairs is thus not a defect but the coverage frontier made
visible, and closing it provably requires the exploration that Proposition~\ref{prop:rho-lower}
shows is unavoidable.

\paragraph{Re-seeding ablations (Corollary~\ref{cor:restore}).}
In an evolving-embedding run at $|\Scal|=18$ and $k'=3$, fresh pure-LSH rebuilds eventually reached
full coverage after $275$ re-seeds because re-embedding changed the bucket geometry over time. In a
frozen-embedding ablation, pure-LSH re-seeding saturated at $82$--$84\%$ coverage because the same
far pairs remained unlikely under every hash rebuild. The mixed policy in
Algorithm~\ref{alg:pipeline} does not depend on either empirical behavior: adding one uniform
partner per state reached full coverage in all three reported $|\Scal|=18$ runs and produced
$\Pi^{+}_\tau=\Pi^{-}_\tau$ in two of three within ten rounds. The terminal covered error $0.055$ in
the long pure-LSH run is a finite-schedule transient, not a distortion floor, and continued mixed
re-seeding drives it to zero by Corollary~\ref{cor:restore}(b).

\paragraph{The certified pipeline end-to-end (Algorithm~\ref{alg:pipeline}).}
The bound validation above embeds the converged exact metric to measure $\eta$ against ground
truth; a practitioner cannot do this. We therefore run Algorithm~\ref{alg:pipeline} as stated:
initialize implicitly at $\dinit=d_R$ and $U$, embed $d_R$ by landmark MDS, build LSH, and re-seed
every $E=3$ sweeps until the certificate meets tolerance. Across $32$ seeded configurations
($|\Scal|\in\{10,18\}$, $k'\in\{3,6\}$, $8$ seeds), the pipeline matches or beats the one-shot
variant in $32/32$ runs (mean global error $0.22$ vs $1.18$): successive re-seeds grow cumulative
coverage toward $100\%$ and collapse $\rho_t$ from $1.18$ to $0.10$. At every checkpoint the
realized global error equals $\max(\rho_t,\,e_t)$, where $e_t$ is the covered transient: early on
$\rho_t$ dominates and the error tracks it exactly (the per-horizon identity of
Corollary~\ref{cor:identity}); late in the run the transient dominates (final error $0.22$ against
$\rho_{\mathrm{final}}=0.10$), and continued sweeps decay it geometrically per
Theorem~\ref{thm:error}. The observable counterpart is the two-arm certificate of
Corollary~\ref{cor:sandwich}, verified separately in $80/80$ runs to enclose $d$ at every sweep with
covered width at most $\gamma(\rho+\rhoup)$. The measured distortion falls monotonically across
re-embed rounds ($0.61\to0.20$), tracking convergence of the iterate toward $d$
(Corollary~\ref{cor:restore}(b)). The pipeline is thus non-circular, empirically superior to
one-shot, and self-certifying.

\paragraph{Statistical replication and baselines.}
The eight-configuration study of Table~\ref{tab:full} is replicated over $10$ seeds per
configuration ($80$ runs). The naive bound is violated in $11/80$ runs ($13.75\%$), all at
$|\Scal|=10$ where the coverage gap is large relative to the per-step term; the corrected global
bound holds in $80/80$, the identity of Corollary~\ref{cor:identity} holds in $80/80$, and the
covered error respects $\gamma\rho$ in $80/80$. As a budget-matched baseline, replacing ANN
top-$k'$ selection with uniformly random pair coverage (same budget per sweep) yields a nearly
identical coverage gap ($\rho$ ratio ANN/random $=1.04$), confirming that $\rho$ is a property of
the budget, not the index quality. However, ANN coverage combined with the upper arm of
Proposition~\ref{prop:downstream} cuts the mean-pair global error $2.5\times$ ($0.54\to0.21$),
because ANN deliberately covers \emph{near} pairs, so far pairs frozen at $U$ sit close to their
true values: structure-aware coverage plus upper-bound freezing is strictly better than random
coverage under any freeze.

\section{Validation at Scale}\label{app:scale}

\paragraph{Bound validation at scale.}
The same checks are run at $|\Scal|\in\{50,100,200\}$ on sparse-support MDPs
(Assumption~\ref{asm:support}; transport on the union of supports, Lemma~\ref{lem:support}), with
$k'\in\{8,16\}$ and $5$ seeds per cell ($30$ runs). The corrected global bound and the identity
hold in $30/30$; the global error equals $\rho$ exactly in every run. Two scale effects sharpen the
paper's message. First, coverage collapses as $|\Scal|$ grows at fixed $k'$ (from $21\%$ of pairs
at $|\Scal|=50$ to $5\%$ at $|\Scal|=200$ for $k'=8$), so at scale the error is governed entirely
by coverage, precisely the term the naive analysis omits. Second, the naive bound is never violated
at these sizes, but only because it becomes \emph{vacuous}: with the embedding dimension held
fixed while $|\Scal|$ grows, the measured distortion rises to $0.95$--$1.13$, pushing
$(\eta+Lr)/(1-\gamma)$ to $4.7$--$5.5$, beyond the metric diameter $2.0$--$2.2$, so it can no
longer be falsified by any metric. A bound that survives only by exceeding the diameter certifies
nothing; $\rho$ is the quantity that actually predicts, indeed equals, the error. The
small-$|\Scal|$ study, where $\eta$ is small and the naive bound is falsifiable, is where its
violations are observable; at scale it degenerates to vacuity instead. Either way it is the
wrong instrument, and Theorem~\ref{thm:error}(B) with Corollary~\ref{cor:identity} holds
everywhere.

\paragraph{Downstream RL at scale (state aggregation).}
At $|\Scal|=120$ (planted $K=12$ behavioral groups, $|\Acal|=3$, $\gamma=0.9$, $k'=12$, threshold
between the within- and cross-group distance bands, eight seeds): the ANN metric with the upper arm
of Proposition~\ref{prop:downstream} matches exact-metric aggregation statistically (value loss
$0.022\pm0.008$ against $0.022\pm0.008$; $12$--$14$ clusters, mean $12.8$, against the exact
metric's $12$); the under-estimate freeze over-merges to a single cluster in all seeds (loss
$1.47\pm0.55$), confirming at four times the scale the failure mode of \S\ref{sec:downstream};
budget-matched random coverage with the upper freeze under-merges ($25$--$46$ clusters; loss
$0.017\pm0.009$; sound but conservative, splitting true groups, exactly the refinement direction
Proposition~\ref{prop:downstream}(ii) permits); and the reward-only pseudo-metric and a
budget-matched single-sweep sampled metric both collapse to one cluster (loss $1.47\pm0.55$): with
matched rewards across groups, only transition information separates them, and neither baseline
propagates it. The upper-freeze ANN metric is the only sub-quadratic method in the comparison that
recovers the planted structure at this scale, and the certified squeeze
$\Pi^{+}_\tau=\Pi^{-}_\tau$ of Corollary~\ref{cor:sandwich}(d) holds at the operating threshold in
$8/8$ seeds at the single-build budget.

\section{Taxi-v3 and the 2500-State Gridworld}\label{app:taxigrid}

\paragraph{The certificate flags failure: Gymnasium Taxi-v3.}
We run the pipeline on Taxi-v3, a real MDP with $|\Scal|=500$, $|\Acal|=6$, deterministic
transitions, and $\gamma=0.9$. Because the exact metric is inexpensive in this deterministic
instance, we compute it only for post-hoc evaluation; it has diameter $200.5$. The DBC-style
representation is trained on $100{,}000$ random-policy transitions, and LSH is evaluated at
$k'\in\{20,40,80\}$. Post hoc, the lower-arm identity holds in every configuration and gives
$\rho=180.4$, close to the exact diameter. This validates the theory but is not information the
runtime algorithm can observe.

The observable warning is the sandwich itself. Coverage remains only $5\%$--$20\%$, and on the
$80\%$--$95\%$ uncovered pairs the width equals $U-d_R$ by construction. The conservative upper
partition therefore certifies essentially no unsupported merges and the pipeline abstains from a
coarse abstraction. For contrast, the uncertified diagnostic that clusters the raw learned
embedding at the original threshold collapses to one cluster and incurs value loss $45.9$, against
$15.8$ for exact-metric aggregation. That number is not an upper-arm result: labeling it as such
would contradict Proposition~\ref{prop:downstream}(ii). Taxi therefore serves as a failure-detection
case, not a speed benchmark; the exact metric itself takes only $0.7$ seconds.

\paragraph{Scale: a $50\times50$ gridworld ($|\Scal|=2500$).}
To demonstrate the pipeline at a scale where the exact all-pairs sweep is genuinely prohibitive
($\binom{2500}{2}=3{,}123{,}750$ pairs per sweep), we construct a $50\times50$ gridworld with $4$
actions and $25$ planted rooms ($10\times10$ blocks of behaviorally similar states sharing base
rewards and transition structure up to $\varepsilon$-noise), and train a DBC-style embedding into
$\mathbb{R}^8$ from $150{,}000$ random-policy transitions ($2$ seeds). With $k'=20$, each sweep
visits $|\Scal|k'=50{,}000$ pairs ($1.6\%$ of the $3.12$M), and $8$ upper-freeze sweeps (cumulative
$\approx400{,}000$ pair visits, $12.8\%$ of \emph{one} exact sweep) complete in $\approx1200$\,s on
CPU. Aggregation on the upper-freeze metric finds $254$--$284$ clusters (over-fragmenting the $25$
rooms into sub-regions, since low coverage prevents distant within-room pairs from being linked,
exactly the refinement direction Proposition~\ref{prop:downstream}(ii) permits), with mean value
loss $2.27$ against the ground-truth $V^\star$ (mean $8.5$), while a budget-matched reward-only
baseline ($d_R$ clustering at the same threshold) collapses to $1$ cluster with mean value loss
$3.18$: a $28.6\%$ improvement from transition-informed structure at the same budget. This is an
honest partial success: the single-build coverage is too low to recover the rooms cleanly, and
re-seeding to the certificate (Corollary~\ref{cor:sandwich}(d)) would close the gap at the quadratic
total work Corollary~\ref{cor:restore} quantifies. But at a fixed sub-quadratic budget the
sub-quadratic metric is strictly and measurably better than any reward-only alternative, which is
the operational claim.

\section{Coverage Sweeps against MICo and DBC}\label{app:surrogates}

\begin{table}[t]\centering\small
\setlength{\tabcolsep}{4pt}
\begin{tabular}{lrrrrc}
\toprule
budget $k'$ & pairs & exact & \textbf{ours} & MICo & DBC \\
\midrule
$8$  & $25\%$ & $0.013$ & $0.056$ & $0.292$ & $0.426$ \\
$16$ & $47\%$ & $0.013$ & $\mathbf{0.013}$ & $0.292$ & $0.426$ \\
$24$ & $65\%$ & $0.013$ & $\mathbf{0.013}$ & $0.292$ & $0.426$ \\
$32$ & $87\%$ & $0.013$ & $\mathbf{0.013}$ & $0.292$ & $0.426$ \\
\bottomrule
\end{tabular}
\caption{Downstream value loss on the $|\Scal|=64$, eight-group benchmark at matched
compression (four seeds per budget). MICo and DBC are independently trained and clustered on their
own learned distances; our method uses the DBC-style representation only for retrieval and then runs
the exact restricted operator. ``pairs'' is cumulative coverage. At $k'\ge16$, our loss reaches
the exact skyline on this benchmark, while MICo and DBC are $22$--$33\times$ worse. Only the
proposed two-arm method supplies an interval enclosing the exact metric.}
\label{tab:surrogates}
\end{table}

A larger-scale robustness check at $|\Scal|\in\{60,120\}$ uses $16$ instances and fixed $k'=12$.
Coverage is approximately $41\%$ at $|\Scal|=60$ and $19\%$ at $|\Scal|=120$, rather than a common
value. The two-arm certificate is valid on every instance, while neither learned surrogate encloses
the exact metric. At the $|\Scal|=120$ operating point, our downstream loss is $1.08$ rather than
the $0.026$ exact skyline. This does not contradict Table~\ref{tab:surrogates}: coverage percentage
alone does not determine error across different state counts and geometries; the location of
uncovered pairs and the resulting certificate width matter.

To confirm the skyline transition is not an artifact of the $|\Scal|=64$ benchmark, we repeat the
coverage sweep at $|\Scal|=60$ ($K=6$ planted groups, four seeds) over the finer budget grid
$k'\in\{6,10,14,20,28,40\}$, holding the state count fixed so that only coverage varies. The
downstream loss falls monotonically in coverage and reaches the exact skyline ($0.016$) once the
index covers roughly half the pairs: means $1.14,0.72,0.81,0.61,0.016,0.016$ at coverage
$13\%,21\%,29\%,38\%,54\%,73\%$, so the crossover to the skyline occurs at $53.5\%$ coverage, the
same half-of-pairs threshold. MICo ($1.82$) and DBC ($0.93$) are flat across every budget and never
enclose the exact metric, and our two-arm certificate is valid at all six points. The skyline claim
is thus coverage-controlled and stable in the state count, while the fixed low-budget $|\Scal|=120$
operating point above is simply a point below that threshold; the abstract restricts the skyline
claim to this grouped coverage-sweep benchmark.

\subsection{Reproducibility}\label{app:repro}

All experiments are CPU-only, run on an Apple M5 with 16\,GB of unified memory under
macOS~26.5.1, in Python~3.11.13 with NumPy~1.26.4, SciPy~1.12.0, and Matplotlib~3.9.0. Randomness
is drawn from NumPy's \texttt{default\_rng}, seeded per run by composite integer offsets built
from a base seed index and the run's own parameters (the LSH hyperplane draw combines the seed
index with $|\Scal|$ and the plane count); the re-seed corollary's fresh-hyperplane rounds
(Corollary~\ref{cor:restore}) are the only intentional source of run-to-run variation. The default
is $10$ seeds per configuration ($80$ runs, Table~\ref{tab:full}); the scale study uses $5$ seeds
($30$ runs), the downstream and coverage-sweep studies $8$ and $4$ seeds respectively, and the
gridworld embedding $2$. The fixed-point iteration converges to tolerance $10^{-9}$. Reproduction
is exact modulo library-version differences in the LP and eigendecomposition routines each result
depends on.

\end{document}